\UseRawInputEncoding

\documentclass[10pt]{article}
\usepackage[letterpaper, left=1in, top=1in, bottom=1in, verbose, ignoremp]{geometry}

\usepackage{url}
\usepackage{latexsym,amssymb,amsfonts,fancyvrb,amsthm,enumerate,subcaption,mathrsfs}
\usepackage[numbers]{natbib}
\usepackage{placeins}
\usepackage{caption}
\usepackage{textcomp}
\usepackage{chngcntr}
\usepackage{amscd,epsfig}
\usepackage{enumitem}

\usepackage{authblk} 
\usepackage{setspace}

\usepackage{mathtools} 
\usepackage{booktabs} 
\usepackage{tikz} 
\usetikzlibrary{shapes.geometric, positioning, calc}
\tikzset{%
    mynode/.style={%
        circle, radius=2pt, draw=darkgray, fill=white
    }
}

\makeatletter
\def\blfootnote{\gdef\@thefnmark{}\@footnotetext}
\makeatother

\usepackage{tikz-cd}
\usepackage[linesnumbered,ruled]{algorithm2e}

\usepackage{hyperref}

  \newtheorem{theorem}{Theorem}[section]
  \newtheorem{proposition}[theorem]{Proposition}
  \newtheorem{lemma}[theorem]{Lemma}

  \newtheorem{definition}{Definition}[section]
  
  \newtheorem{example}{Example}[section] 
 \newtheorem{remark}{Remark}[section]

\newcommand{\MM}{\mathcal{M}}

\newcommand{\I}{\mathcal{I}}

\newcommand{\id}{\mathrm{I}}

\newcommand{\R}{\mathbb{R}}

\newcommand{\indep}{\perp \!\!\! \perp}

\newcommand{\minus}{\scalebox{0.75}[1.0]{$-$}}
\DeclareMathOperator*{\argmax}{arg\,max}
\DeclareMathOperator*{\argmin}{arg\,min}

\def\newop#1{\expandafter\def\csname #1\endcsname{\mathop{\rm
#1}\nolimits}}

\newop{Inv}
\newop{conv}
\newop{pa}
\newop{de}
\newop{an}
\newop{nd}
\newop{ch}
\newop{CS}
\newop{diag}
\newop{Var}
\newop{top}

\title{Learning Linear Gaussian Polytree Models with Interventions}
\author[1]{D. Tramontano}
\author[1]{L. Waldmann}
\author[1]{M. Drton} 
\author[2]{E. Duarte}
\affil[1]{Technical University of Munich; School of Computation, Information and Technology, Department of Mathematics and Munich Data Science Institute, Germany}
\affil[2]{Faculdade de Ciências, Universidade do Porto, Portugal}
\date{}

\begin{document}
\maketitle

\begin{abstract}
We present a consistent and highly scalable local approach to
learn the causal structure of a linear Gaussian
polytree using data from interventional experiments with known intervention targets. Our methods first
learn the skeleton of the polytree  and then orient its edges. The output is a CPDAG representing the interventional equivalence class of the
polytree of the true underlying distribution.
The skeleton and orientation recovery procedures we
use rely on second order statistics and low-dimensional marginal distributions.
We assess the performance of our methods
under different scenarios in synthetic data
sets and apply our algorithm to learn a polytree 
in a gene expression interventional data set. Our simulation studies demonstrate that our approach is fast, has good accuracy in terms of structural Hamming distance, and 
handles problems with thousands of nodes.
\end{abstract}

\blfootnote{“© 2023 IEEE.  Personal use of this material is permitted.  Permission from IEEE must be obtained for all other uses, in any current or future media, including reprinting/republishing this material for advertising or promotional purposes, creating new collective works, for resale or redistribution to servers or lists, or reuse of any copyrighted component of this work in other works.”}
\section{Introduction}\label{sec.intro}
The dominant approach in recent literature on causal discovery from interventional data is optimization of a model score.
Although the scoring is straightforward in the sense that the  optimization 
over DAGs 
refers to fully specified joint models for all
observational and interventional data, the optimization landscape is very high-dimensional, making score-based algorithms infeasible for graphs with hundreds/thousands of nodes that are common in biological applications. This makes causal discovery  difficult, in addition to many other challenges that remain such as
departing from restrictive genericity assumptions on the underlying distributions and developing methodology for high-dimensional settings. 
In this article, to address these challenges, we depart from a score-based strategy and leverage special properties of polytrees to
obtain a highly scalable ``local'' approach that learns from low-dimensional 
marginals. 

Our methods yield fast and consistent algorithms to learn linear Gaussian polytrees from interventional data. 
 The skeleton is learned by aggregating pairwise correlations from different experimental settings;  
 edge orientations are found by testing pairwise regression coefficients on suitable subsets of the data.  
This removes the need to form a score that contemplates joint 
models for all data.  Moreover, we allow the intervention targets to be arbitrary subsets (with no variable always intervened upon). We are not aware of any other work with these features. While it will be interesting to seek extensions to 
broader classes of graphs in future work, we stress that for
 very high-dimensional problems (e.g., \citet{dixit:2016}) it is of interest to target simpler computationally tractable objects that may be inferred reliably with moderate sample size \citep{edwards:2010}.  For this reason, polytrees have received renewed interest.

\section{Related work}
Directed acyclic graphs (DAGs) have been extensively used in causal modeling; the 
nodes of a graph represent the random variables of the model while the directed 
edges represent causal effects from source to sink. The effects of the parent nodes 
on the children are quantified by structural equations. Causal discovery is then 
the problem of inferring the graphical structure underlying the 
model. 

Approaches for causal discovery using only observational data and under the assumption that all variables in the model are directly
observed/measured, include 
constraint-based, score-based and hybrid methods (e.g., PC-algorithm 
\citep{sprites:caus:pred}, Greedy Equivalent Search (GES) \citep{chickering:2002}, 
Greedy SP algorithm \citep{solus:2017}).
Without extra assumptions on the data 
generating process, these methods learn a completed
partially directed graph (CPDAG), a mixed graph that encodes the causal information 
common to all the members of a Markov equivalence class (MEC). 
Classical constrained-based algorithms, such as the PC algorithm, 
can suffer from the elevated number of conditional independence tests that are needed to learn the CPDAG. A recent line of work \citep{akbari:2021,mokhtarian:2021,mokhtarian:2023}, which has proven to be almost optimal in terms of the number of conditional independence tests performed and is also applicable in the presence of unobserved variables, exploits the idea of learning the graph recursively starting from Markov boundary information.

Learning only the CPDAG is not always satisfactory as DAGs in the same MEC can have 
opposite causal interpretations. However, using additional assumptions such as non-
Gaussianity \citep{shimizu:hoyer:2006}  and/or non-linearity 
\citep{hoyer:additive:2008} it is possible to identify the 
complete causal structure. Whenever these assumptions do not apply, such as in the
linear Gaussian case, additional data from interventional experiments can help to
improve the identifiability of directed edges 
by refining the MEC.
This refinement is quantified in terms of interventional
Markov equivalence classes ($\I$-MECs) \citep{YKU18}. An $\I$-MEC is a collection of DAGs that entails the 
same interventional distributions for a fixed choice of intervention targets $\I$.

Within the causal discovery methods that use interventional data, there are those in which intervention targets are known such as Greedy Interventional Equivalent Search (GIES)
\citep{Hauser:20212}, Interventional GSP (IGSP) \citep{wang:2017}, and Joint Causal Inference \citep{Mooij:2020}. Other methods accommodate for unknown interventions
targets, for instance Differentiable Causal Discovery with Interventions (DCDI) \cite{brouillard:2020}, permutation based approaches \cite{squires:2020}, and Bayesian Causal Discovery with unknown Interventions (BaCaDI) \cite{hagele:2022}. For a recent review
of causal discovery methods we refer the reader to
\cite{squires:2022,zanga:2022}.

In this paper, we address the problem of causal discovery when
both observational and interventional data are available and all variables are measured. 
Our focus is on linear Gaussian structural causal models in which the graph is a polytree and the 
intervention targets are known. As shown computationally in \citep{acid:1995},
the polytree assumption 
provides an effective compromise between computational complexity and model expressiveness.  This 
property of polytrees has been effectively exploited in image segmentation \citep{fehri:2019}, 
hardware optimization \citep{zaveri:2010}, and Ozone prediction \citep{sucar:1997}.  
The polytree assumption follows a recent paradigm in the causal discovery literature in which assumptions are  made about the DAG underlying the causal model in order to reduce the complexity of learning algorithms.  Other methods that follow a similar approach include the causal additive trees
(CAT) method which assumes the underlying DAG is a tree \citep{jakobsen:2022} and the method from \cite{mokhtarian:2022} which
incorporates side information such as the assumption that the ground truth is a diamond free graph or  that an upper bound on the clique number of the graph is known.

The polytree assumption has been studied since the early days of causal reasoning theory. Indeed,
\citet{rebane:pearl:1987} use the Chow-Liu algorithm \citep{chow:1968} to learn the skeleton of a polytree. 
Different variants of the Rebane and Pearl approach that work under different sets of assumptions have been developed in \citep{lou:2021,tramontano:2022,chatterjee:2022}, a linear programming algorithm is developed in \citep{linusson:2022}, while in \cite{Azadkia:2021} the graph is assumed to be locally a polytree around a targeted node allowing to infer the directed causes of the target node. Both \citet{etesami:2016}, in the context of time series graphs, and \citet{sepher:2019} in the context of classical graphical models, introduce a notion of minimality for polytrees with hidden nodes and provide an algorithm for learning the graph under the assumption of minimality. Polytree learning has been proven to be an NP-hard problem in \citep{dasgupta:2013}, the complexity of the problem is studied in full details in \cite{Gruttemeier:2021}. 
In \cite{amendola:2021}, a complete characterization is given for the constraints that emerge between 2nd and 3rd order moments of a random vector that follows a polytree-based linear structural equation model with non-Gaussian error terms.

\section{Preliminaries}

\subsection{Notation for Graphs} \noindent
A \emph{directed graph} is a pair $G=(V,E)$, where $V$ is the set of vertices and $E\subset \{(u,v): u, v\in V,\, u\neq v\}$ is the set of directed edges. We denote a pair $(u,v)\in E$ also by $u\to v$. 
A \emph{walk} from node $v$ to node $w$ in $G$ is an alternating sequence $(v_0,e_1,v_1,e_2,\dots,v_{k-1},e_k,v_k)$ consisting of nodes and edges of $G$ such that $v_0=v$, $v_k=w$, and $e_l\in\{(v_{l-1},v_l),(v_{l},v_{l-1})\}$ for $l=1,\dots,k$. A walk is a \emph{directed path} if $e_i=(v_{i-1},v_{i})$ for all $i\in \{1,\dots,k\}$ and a \emph{directed cycle} if additionally 
$v_0=v_k$.  
From now on we assume that the graph $G$ is a \emph{DAG} (directed acyclic graph), i.e., $G$ does not contain any directed cycles.

 A node $v_l$ is a \emph{collider} on a walk as above if $e_{l-1}=(v_{l-1},v_l)$ and $e_l=(v_{l+1},v_l)$. Moreover, $v_l$ is an \emph{unshielded collider} if neither $(v_{l-1},v_{l+1})$ or $(v_{l+1},v_{l-1})$ belongs to $E$. 

A walk that does not contain a collider is called a \emph{trek} from $v$ to $w$.  Every trek contains a unique node $v_l$ that splits the trek into two directed paths from $v_l$ to $v$  and from $v_l$ to $w$, respectively.  This node is the \emph{top} of the trek.  Note that the top may be equal to $v$ or $w$, in which case one of the two directed paths is trivial consisting of a single node and no edge.  A trek is \emph{simple} if it does not contain repeated nodes.

If $u\to v\in E$, then $u$ is a \emph{parent} of $v$, and $v$ is a \emph{child} of $u$. If $G$ contains a directed path from $u$ to $v$, then $u$ is an \emph{ancestor} of $v$ and $v$ is a \emph{descendant} of $u$. The set of parents, children, ancestors, and descendants of $u$ are denoted by $\pa(u),\ch(u),\an(u),\de(u)$, respectively.
The \emph{skeleton} of a DAG is the undirected graph obtained by replacing each edge $(u,v)$, by an undirected edge, denoted here by $\{u,v\}$.

A \emph{mixed graph} is a triple $G = (V, E, U)$, where $E$ is the set of directed edges defined as above, and $U\subset E\subset \{\{u,v\}: u, v\in V,\, u\neq v\}$ is the set of undirected edges. We assume all the graphs we consider to be \emph{simple}, i.e.,  there is at most one edge, directed or undirected, between any two vertices.

\subsection{Linear Structural Causal Models} \label{subsec.causal.poly}
\noindent
 Let $X=(X_u)_{u\in V}$ be a random vector indexed by the vertices of a DAG $G$. For $A\subset V$, let $X_A=(X_u)_{u\in A}$.  
 When $X_A$ is conditionally independent of $X_B$ given $X_C$ for disjoint subsets $A,B,C\subset V$, we write $A\indep B|\,C$.
 The joint distribution of $X$ satisfies the \emph{local Markov property} with respect to $G$ if
$
    \{i\} \indep [p]\setminus(\pa(i)\cup \de(i))\ |\ \pa(i)\ \forall\ i\in[p].
$
The Markov equivalence class of $G$ is the set of all DAGs that encode the same conditional independence relations, i.e., for which the set of distributions satisfying the local Markov property is the same.  See \cite[Chap.~1]{handbook} for further details.

The Gaussian structural causal model given by $G$ postulates that
\begin{align}
\label{eq.structural.equations}
    X_v &= \sum_{w\in\pa(v)} \lambda_{wv} X_w + \varepsilon_v, \qquad v\in V,
\end{align}
where the edge coefficients $\lambda_{wv}\in\mathbb{R}$ are unknown parameters and the errors $(\varepsilon_v)_{v\in V}$ are independent Gaussian random variables.  Each error is assumed to have mean zero and unknown variance $\omega_v>0$; in symbols, $\varepsilon_v\sim\mathcal{N}(0,\omega_v)$.
Let $\Lambda\in\mathbb{R}^{V\times V}$ be the matrix of edge coefficients, with zeros filled in at non-edges.  Let $\Omega=\diag((\omega_v)_{v\in V})\in\mathbb{R}^{V\times V}$ be the diagonal covariance matrix of $\varepsilon=(\varepsilon_v)_{v\in V}$.  Solving~\eqref{eq.structural.equations}, we obtain that $X=(\id-\Lambda^T)^{-1}\varepsilon$ is a Gaussian random vector with covariance matrix
\begin{align}
   \label{eq.Sigma}
   \Sigma:= \Var[X] \;=\; (\id-\Lambda^T)^{-1}\Omega(\id-\Lambda)^{-1},
\end{align}
we stress that the matrix $(\id-\Lambda)$ is always invertible when the graph $G$ is acylic, this is because the matrix $\Lambda$ is strictly lower triangular, so the determinant of $(\id-\Lambda)$ is one.

As presented, $X$ is modeled to have mean zero.
This is without of loss of generality for the later results which solely pertain to the covariance structure.
In the sequel, we denote the space of matrices supported on the edge set of $G$ as
\[
\mathbb{R}^E=\{\Lambda\in\mathbb{R}^{V\times V}: (v,w)\notin E\implies \Lambda_{vw}=0\}.
\]
We write $D_+$ for the set of diagonal matrices in $\R^{V\times V} $ with positive diagonal entries.
The covariance model induced by the DAG $G$ is  the set of positive definite matrices
\begin{align*}
\mathcal{M}(G) \::=\:
\left\{ (\id-\Lambda^T)^{-1}\Omega(\id-\Lambda)^{-1} : \Lambda\in\mathbb{R}^E,\; \Omega\in D_+\right\}.
\end{align*}
Two graphs $G_1$ and $G_2$ are \emph{Markov equivalent} if and only if $\mathcal{M}(G_1)=\mathcal{M}(G_2)$ (see e.g.~\cite[Thm.~8.13]{richardson:2002}).
Combinatorially, the MEC is represented
by its CPDAG
\cite[Chap.~1]{handbook}.

Writing explicitly the entries of the a covariance matrix $\Sigma\in\mathcal{M}(G)$ using Eq.~\eqref{eq.Sigma}, one can get two useful parametrizations for the set $\mathcal{M}(G)$, we report the two parametrizations here, and refer to \cite{sullivant:2010} and the references therein for further details.
\begin{proposition}[Trek-rule]
\label{prop.trek.rule}
  Let $\mathcal{T}(v,w)$ be the set of all treks from $v$ to $w$.  The matrix $\Sigma$ from Eq.~\eqref{eq.Sigma} has its entries 
  \[
  \Sigma_{vw} = \sum_{\tau\in\mathcal{T}(v,w)} \omega_{\top(\tau)}\prod_{e\in\tau} \lambda_e,  \quad v,w\in V.
  \]
  Moreover, the entries of $\Sigma$ satisfy the recursive relation
  \[
  \Sigma_{vw} = \sum_{\tau\in\mathcal{S}(v,w)} \Sigma_{\top(\tau),\top(\tau)}\prod_{e\in\tau} \lambda_e,  \quad v,w\in V,
  \]
  where $\mathcal{S}(v,w)$ is the set of simple treks from $v$ to $w$.
\end{proposition}

Let $\Sigma=(\sigma_{vw})$ be the covariance matrix of a random vector $X$, with diagonal entries $\sigma_{vv}>0$.  The correlation matrix $R(\Sigma)=(\rho_{vw})$ of $\Sigma$ is the matrix with entries $\rho_{vw}=\sigma_{vw}/\sqrt{\sigma_{vv}\sigma_{ww}}$.  It
is also the covariance matrix of the standardized random vector $(X_v/\sqrt{\sigma_{vv}})_{v\in V}$.

\begin{proposition}[Correlation matrices]
\label{prop.correlations}
  If $\Sigma$ is a covariance matrix in $\mathcal{M}(G)$, then its correlation matrix $R(\Sigma)$ is also in $\mathcal{M}(G)$.  Hence, there exists $\Lambda=(\lambda_{vw})\in\mathbb{R}^E$ such that 
  \[
  R(\Sigma)_{vw} = \sum_{\tau\in\mathcal{S}(v,w)} \prod_{e\in\tau} \lambda_e,  \quad v,w\in V.
  \]
\end{proposition}
\begin{proof}
  Write $\Sigma=(\id-\Lambda')^{-T}\Omega'(\id-\Lambda')^{-1}$ for $\Lambda'\in\mathbb{R}^E$ and $\Omega'\in D_+$.  Then $R(\Sigma)=(\id-\Lambda^T)^{-1}\Omega(\id-\Lambda)^{-1}$, where the entries of $\Lambda\in\mathbb{R}^E$ are $\Lambda_{vw}=\Lambda'_{vw}\sqrt{\Sigma_{vv}}/\sqrt{\Sigma_{ww}}$ and those of $\Omega\in D_+$ are $\Omega_{vv}=\Omega'_{vv}/\Sigma_{vv}$.
  
  The second assertion follows from the simple trek-rule in Proposition~\ref{prop.trek.rule}, observing that $R(\Sigma)_{vv}=1$ for all $v$.
\end{proof}

\subsection{Interventions}\noindent
\label{subsec.ICPDAG}
It is useful to formally consider  the collection of linear Gaussian structural
causal models that arise from $G$ and a set of interventions.
 In a \emph{soft intervention}, a subset $I\subset V$ of target nodes is selected, and for each $v\in I$ the conditional distribution of $X_v$ given $X_{\pa(v)}$ is modified. When $X_v$ is made independent of its parents, the intervention is 
 \emph{perfect}.  It is common that different interventions are performed and hence we have a collection of intervention targets denoted by $\I$, so $\I\subseteq 2^V$. Without loss of generality we assume that $\emptyset \in \I$ and refer to this as the observational experiment.  Theorem 3.14 in \citet{YKU18} justifies this assumption, which subsumes the case of conservative intervention targets treated by \citet{Hauser:20212}.

We assume that each interventional experiment obeys a linear Gaussian structural causal model defined by $G$ (recall Section~\ref{subsec.causal.poly}).
Namely, for each $I\in \I$, we have a random vector  $X^{I}:=(X_v^{(I)})_{v\in V}$ with structural equations
\begin{align}
\label{eq:struct}
X_{v}^{(I)}=\sum_{w\in \pa(v)}\lambda_{wv}^{(I)}X_{w}^{(I)}+\varepsilon_{v}^{(I)}, \qquad v\in V.
\end{align}
We use $\Lambda^{(I)}$ to denote the matrix of edge coefficients for this
model and $\Omega^{(I)}=\diag((\omega_{v}^{(I)})_{v\in V})$ to denote
the covariance matrix of the Gaussian vector $\varepsilon^{(I)}:=(\varepsilon_v^{(I)})_{v\in V}$. To encode the invariances of the structural equations of nodes that are not intervened on, i.e are not in $I$, we impose that $\lambda_{wv}^{(I)}=\lambda_{wv}^{(\emptyset)}$ and
$\omega_{v}^{(I)}=\omega_{v}^{(\emptyset)}$ whenever
$v\notin I$.
Each $I$  induces a covariance model
\[
\MM(G,I)=\{\Sigma^{(I)}:\Lambda^{(I)} \in \mathbb{R}^{E},
\Omega^{(I)} \in D_{+}\}
\]
where $\Sigma^{(I)}= (\id-(\Lambda^{(I)})^{T})^{-1}\Omega^{(I)}(\id-\Lambda^{(I)})^{-1}$  and
$ \Lambda^{(I)},\Omega^{(I)}$ satisfy the invariances, on edge coefficients and error variances, of nodes that are not intervened on. 

The interventional DAG model specified by $G$ and the set of intervention targets
$\I$ is the set
\[
\MM_{\I}(G):=\{(\Sigma^{(I)})_{I\in \I}: \Sigma^{(I)}\in \MM(G,I), I\in \I\};
\]
it consists of tuples
 of covariance matrices of length $|\I|$ that may arise by performing interventions according to $\I$.

Two DAGs $G_1,G_2$ are in the same $\I$-\emph{Markov equivalence class}, $\I$-MEC, if and only if $\MM_{\I}(G_1)=\MM_{\I}(G_2)$.
To decide if two DAGs are in the same $\I$-MEC, \citet{YKU18} introduce the notion of an $\I$-DAG:
\begin{definition}[$\I$-DAG] Fix a collection of interventions $\I$ and a DAG $G$, the $\I$-DAG $G^\I$ is the graph $G$ augmented with $\I$-vertices $\{\zeta_I\}_{\emptyset\neq I\in\I}$, and the $\I$-edges $\{\zeta_I\to u\}_{u\in I,I\in\I}$.
\end{definition}
The following theorem provides a concise graphical representation of the $\I$-\emph{Markov equivalence classes}.
\begin{theorem}{\cite[Thm.~3.14]{YKU18}}
Let $\I$ be a conservative set of intervention targets. Two DAGs $G_1,G_2$ are in the same $\I$-MEC
iff for all $I\in\I$ the $\I$-DAGs $G_1^{\Tilde{\I}_I}$ and $G_2^{\Tilde{\I}_I}$ have the same skeleton and the same v-structures, where 
\begin{equation*}
\Tilde{\I}_I=\{\emptyset\} \cup \{I\cup J\}_{I,J\in\I, I\neq J}.
\end{equation*}
\end{theorem}
The $\I$-MEC of a DAG $G$ can be represented uniquely by its $\I$-CPDAG, this provides a combinatorial representation of the causal information that we can extract from the interventional data available, often also referred to as the
$\I$-essential graph in the literature. This
is a mixed graph with the same skeleton as $G$, a directed edge, $(u,v)$, if every member of
the $\I$-MEC has that edge with the same orientation, and an undirected edge, $\{u,v\}$, if
there are two DAGs in the $\I$-MEC for which that edge has opposite orientations.
See \citet[Thm.~18]{Hauser:20212} for details in the setting of perfect interventions,
the same construction carries over to the setting of general interventions \cite{YKU18}.

\section{Learning Causal Models on Polytrees with Interventions} \label{subsec.polytrees}
From now on we assume that the DAG $G$ is a polytree, this means that
the skeleton of $G$ is a tree, i.e., a graph in which there is exactly one path between any two nodes.

Our procedure first learns the skeleton.
For this purpose we create a novel interventional version of the 
Chow-Liu algorithm \cite{chow:1968}.
The challenge here lies in constructing a weight matrix that reveals the tree structure, the same way
the correlation matrix on a single observational data set does. As seen in Fig.~\ref{fig.skeleton.flip.intervention}, 
taking simply the correlation matrix of the pooled data does not work. Thus, we introduce the notion of
a $G$-valid weight matrix that suitably captures the tree structure.  We will 
construct $G$-valid weight matrices by aggregating
correlation matrices. 
The aggregation 
does not necessarily produce a correlation matrix, yet we show that this approach yields a consistent procedure to learn the skeleton in 
low- and high-dimensional settings. The approach is detailed in Section~\ref{sec.skeleton}. Its consistency is discussed in  Appendix \ref{app.consistency.skeleton}.

To determine the orientation of the edges we propose and compare several methods described in
Sections~\ref{sec.singleorientation} and \ref{sec.collidersearch}; their consistency is discussed in the Appendix \ref{app.sec.cons}. 
Importantly, all procedures use only correlations and low-dimensional marginal distributions that are efficiently estimable in a low sample regime, see e.g., the concentration inequality given in \cite[Cor.~1]{kalish:2007} in which it is explicit that the effective sample size decreases as the size of conditioning set increases requiring more samples to achieve the same accuracy in the estimation.

\subsection{Learning a Skeleton}\noindent
\label{sec.skeleton}
In seminal work, \citet{rebane:pearl:1987} proved that when a distribution satisfies conditional independence constraints induced by a polytree, then the skeleton of the tree can be recovered using the algorithm for the maximum weight spanning tree of \citet{kruskal:1956}, with weight matrix given by the mutual information between the variables. Notably, the only property of mutual information  needed in this algorithm is the data processing inequality \citep[Thm.~2.8.1]{cover:2006}. This implies that Kruskal's algorithm finds the correct polytree skeleton any time the weight matrix that is used respects the following condition.
\begin{definition}
Given a polytree $G$, a weight matrix $W\in\mathbb{R}^{p\times p}$ is $G$-valid if for every triplet $u-v-w$ in $G$ 
\begin{equation}
\label{g.valid.ineq}
    \min\{W(u,v),W(v,w)\} \ge W(u,w).
\end{equation}
If all inequalities in \eqref{g.valid.ineq} are strict, then $W$ is strictly $G$-valid.
\end{definition}

\begin{lemma}{\cite[Thm.~1]{rebane:pearl:1987}}
If $W$ is \emph{strictly} $G$-valid, then the maximum weight spanning tree of the complete graph over $V$ with weight matrix $W$ is the skeleton of $G$.
\end{lemma}

In a polytree $G$ there is at most one trek between any two vertices. In particular for any triple $u-v-w$, the only possible trek between $u$ and $w$ is the one involving also $v$, so from Proposition~\ref{prop.correlations} we have that $|\rho_{u,w}|=0$ if $u-v-w$ is a collider triple and $|\rho_{u,w}|=|\rho_{u,v}||\rho_{v,w}|$ otherwise. In particular, the absolute (observational) correlations define a $G$-valid weight matrix that is strictly $G$-valid if $G$ is causally minimal.  Similarly, for each $I\in\I$, the associated absolute correlation matrix $|R^I|$ is  a $G$-valid weight matrix.

To efficiently learn from available interventional data, we wish to form
a single weight matrix that encodes the information of all the experimental settings.
To this end, we use an aggregation function $A:(\mathbb{R}^{p\times p})^{|\I|}\to \mathbb{R}^{p\times p}$ that takes a collection of $G$-valid weight matrices and outputs a strictly $G$-valid matrix. One way of obtaining such a function $A$ is to apply the same order-preserving transformation $a:\mathbb{R}^{|\I|}\to \mathbb{R}$ to each vector of correlations $(\rho_{ij}^1,\dots,\rho_{ij}^{|\I|})$.  By order-preserving, we mean that $a(x_1,\dots,x_{|\I|})\leq a(y_1,\dots,y_{|\I|})$ anytime $x_i\leq y_i$ for $i=1,\dots,|\I|$ . Possible choices for $a$ are:
\begin{enumerate}
\label{eq.weights}
    \item $a(\rho_{ij}^1,\dots,\rho_{ij}^{|\I|}))=
    -\sum_{I\in\I}\frac{n_I}{2}\log(1-(\rho_{ij}^I)^2)$,\footnote{This is well defined 
    for absolute correlations in $[0,1)$.}
    \item Weighted mean,
    \item Weighted median,
\end{enumerate}
where $n_I$ is the size of the dataset $I$ and the weights we refer to for the mean and the median are $\frac{n_I}{\sum n_I}$. 

The output of Kruskal's algorithm is the same for  any strictly $G$-valid matrix $W$.
In practice, however, we work with an estimate $\Tilde{W}$ and different aggregation functions can give different results. We discuss the numerical performance of the three aggregation functions in Section~\ref{subsec.sim.skeleton}.

\subsection{Identifiability of Edge Orientations}\noindent
\label{sec.orientation}
The output of our learning methods is an $\I$-CPDAG (Section~\ref{subsec.ICPDAG}) whose construction simplifies for polytrees. The next definition
singles out the edge directions that  can be identified in a polytree. Notice that in a polytree there is only one path between two vertices, thus all the colliders are unshielded.  
\begin{definition}
\label{def.idenitf}
An edge $u\to v\in G$ is $\I$-directly-identifiable if it is either  part of a collider, or there exists $I\in\I$ such that $|I\cap\{u,v\}|=1$. It is $\I$-identifiable if it is either 
$\I$-directly-identifiable, or there is an edge $w\to u\in E$ that is $\I$-identifiable.
\end{definition}

\begin{proposition}
The $\I$-CPDAG of $G$ is the partially directed graph 
that has the same skeleton of $G$, and whose directed edges are the edges of $G$ that are $\I$-identifiable.
\end{proposition}

It is clear from Definition~\ref{def.idenitf} that once we have computed the skeleton of the polytree, the $\I$-CPDAG can be identified by searching the $\I$-directly-identifiable edges first. Then, once we have a triple of the form $u\to v-w$ we can orient $v-w$ testing if the triple forms a collider. 
In the next sections we describe different orientiation strategies: Section~\ref{sec.singleorientation} focuses on single edge orientations, and Section~\ref{sec.collidersearch} focuses on finding colliders. Searching for one type of $\I$-directly-identifiable edge or the other first does not matter on a population level. However, it can make a difference when working with data due to approximation errors inherent to each
possible approach.
A comparison of performance between these approaches is given in Appendix \ref{app.sec.orientation}.

Hereafter, $X_u^{I},X_{v}^{I}$ denote vectors of observed values of the variables $X_{u},X_{v}$ in the interventional setting
$I\in \I$. Entrywise, $X^I_{u}=(X^I_{u,1},\dots,X^I_{u,n_I})$ and $X^I_{v}=(X^I_{v,1},\dots,X^I_{v,n_I})$. The total sample size is then $n:=\sum_{I\in \I}n_I$.
Fix $\I$ and  an edge $\{u,v\}$,  we define 
\begin{align*}
\I_{\dot{v}}&:=\{I\in \I: v\notin I\},& \!\! \I_{\dot{u},v}&:=\{I\in \I: v\in I,u \notin I\},\\
\I_{\dot{u},\dot{v}}&:=\{I\in \I: u,v\notin I\},&
 \!\!\I_{u,\dot{v}}&:=\{I\in \I: u\in I,v\notin I\}.
\end{align*}

\subsection{Learning Single Edge Orientations}\label{sec.singleorientation} \noindent
\subsubsection{Invariance of Regression Coefficients (IRC)}
\label{subsec.common.reg.coeff}
To orient the edge $\{u,v\}$, we assume there exists $I_{0}\in \I$ where no intervention took place, i.e., $I_{0}=\emptyset$. 
Fix $I\in \I_{u,\dot{v}}$. If the true model is $u\to v$, then an intervention on $u$ does not change the regression coefficient of 
$X_{v}^{I}=\lambda_{uv}^{I}X_{u}^{I}+\epsilon^{I}$, 
$\epsilon^{I}\sim\mathcal{N}(0,\sigma_{v|u})$. 
Namely, $\lambda_{uv}^{I}=\lambda_{uv}^{\emptyset}$.
Thus to orient the edge $\{u,v\}$ we test the hypothesis $\mathcal{H}:\lambda_{uv}^{I}=\lambda_{uv}^{\emptyset}$. There are different options for the choice of this test and IRC testing has been used before in causal structure learning  \citep[e.g.,][]{deml:2018,ghassami:2017}.
Here we use an F-test \cite{toyoda:1974}, which is a modified version of the test statistic in \citep{chow:1960} for the case of unequal variances.
We briefly present it here.

Let $e_{I_0}$ and $e_{I}$ denote the vectors of residuals of the regressions $X_{v}^{I}=\lambda_{uv}^{I_0}X_{u}^{I_0}+\epsilon^{I_0}$ and $X_{v}^{I}=\lambda_{uv}^{I}X_{u}^{I}+\epsilon^{I}$ respectively, and let $e$ denote the vector of residuals, which is obtained by pooling the data sets 
$(X_u^{I_0},X_v^{I_0}), (X_{u}^{I},X_v^{I})$. Under the null hypothesis $\mathcal{H}_{I,u\to v}$ the statistic is
\[
\frac{(e^{T}e-e_{I_0}^{T}e_{I_0}-e_{I}^{T}e_{I})/k}{(e_{I_0}^{T}e_{I_0}-e_{I}^{T}e_{I})/f} \sim F(k,f)
\] 
where $k=2$ and $f=\frac{((n_{I_0}-k)\hat{\sigma}_{I_{0}}^{2}+(n_{I}-k)\hat{\sigma}_{I_{0}})^2}{(n_{I_0}-k)\hat{\sigma}_{n_0}^4+(n_{I}-k)\hat{\sigma}_{n_{I}}^4}$.

For each $I\in \I_{u,\dot{v}}$ we
test the hypothesis $\mathcal{H}_{I,u\to v}$ and reject the orientation
$u\to v$ in favor of $v\to u$ if the p-value of any of these tests is below 
$\alpha/|\I_{u,\dot{v}}|$ for some chosen significance level $\alpha$.  This test is possible whenever $\I_{u,\dot{v}}\not=\emptyset$.

In an analogous fashion, changing the roles of $u,v$, if  $\I_{\dot{u},v}$ is nonempty, we may perform a test for the orientation $v\to u$. 
If only one of the tests $u\to v $ or $v\to u$ is possible, then the decision rule to orient the edge $\{u,v\}$ is clear.
In case both tests are possible, we choose to reject the test with the smallest p-value, as long as such p-value  is below the corresponding significance ($\alpha/|\I_{u,\dot{v}}|$ or $\alpha/|\I_{\dot{u},v}|)$.

\subsubsection{Single Edge BIC Score Before Collider Search}
\label{subsec.likelihood.comp}

Let $\mathcal{H}_1$ and $\mathcal{H}_2$ be any two parametric models  with parameter spaces $\mathcal{M}_1\subset\R^{r_1}$ and $\mathcal{M}_2\subset~\R^{r_2}$ and of dimensions $d_1$ and $d_2$, respectively. 
A model selection based on the BIC consists of choosing the minimizer of the following penalized log-likelihood
\begin{equation}
\label{eq:pen:log}
    l(\mathcal{H}_i\mid X_n)=\argmin_{\theta_i\in\mathcal{M}_i}\big(-\log(p_{\mathcal{H}_i}(X_n\mid \theta_i))+\frac{d_i}{2}\log(n)\big),
\end{equation}
for $i\in\{1,2\}$, where $X_n$ is the observed dataset of size $n$. For more details on model selection and information criteria, we refer the reader to \cite[Chap.~2]{anderson:2004}.

In our setting we want to select between the model in which $u\to v\in G$ and the one in which $v\to u\in G$, respectively denoted as $\mathcal{H}_{u\to v}$ and $\mathcal{H}_{v\to u}$. The polytree assumption allows us to compute the log likelihood using only information on the marginal distribution of $u$ and $v$. We show in detail only the derivation of the penalized log likehood for $\mathcal{H}_{u\to v}$, since the one for the other model is analogous.

Under the model $\mathcal{H}_{u\to v}$ we can write in each dataset $X_v=\lambda_{u,v}^I X_u+\epsilon_{v|u}^I$ where, using Eq.~\eqref{eq:struct}, we can see that 
\begin{align}
    &\epsilon_{v|u}^I=\sum_{w\in \pa(v)\setminus\{u\}}\lambda_{wv}^{(I)}X_{w}^{(I)}+\varepsilon_{v}^{(I)}\sim\mathcal{N}(0,\sigma^I_{v|u}), \\
    &\sigma^I_{v|u}=\sum_{w\in \pa(v)\setminus\{u\}}(\lambda_{wv}^2)^{(I)}\sigma_w^{(I)}+\sigma_{v}^{(I)},    
\end{align}
a consequence of the polytree is assumption is that $X_u$ is independent from all the other parents of $v$, and so it is independent from $\epsilon_{v|u}$ as well. This allows us to rewrite the log-likelihood of a observed pair $(X^I_{u,i},X^I_{v,i})$ as
\begin{align*}
    -\log{2\pi}-\frac{1}{2}\log{{\sigma^I_u}^2}-\frac{1}{2}\log{{\sigma^I_{v|u}}^2}-\frac{(X_{u,i}^I)^2}{2{\sigma^I_u}^2}-\frac{(X^I_{v,i}-\lambda^I_{u,v}X^I_{u,i})^2}{2{\sigma^I_{v|u}}^2}.
\end{align*} 
The log-likelihood for the whole dataset is then
\begin{equation}
    \begin{aligned}
        \label{eq.loglik}
    -n\log{2\pi}+\displaystyle\sum_{I\in\I}\bigg(-\frac{n_I}{2}\log{{\sigma^I_u}^2}-\frac{n_I}{2}\log{{\sigma^I_{v|u}}^2}
    +\displaystyle\sum_{i\in[1,\dots,n_I]}-\frac{(X_{u,i}^I)^2}{2{\sigma^I_u}^2}-\frac{(X^I_{v,i}-\lambda^I_{u,v}X^I_{u,i})^2}{2{\sigma^I_{v|u}}^2}\bigg).        
    \end{aligned}    
\end{equation}
Straightforward computations show that the MLEs for the variance parameters are the usual  variances computed on each dataset individually leading to 
\begin{align}
    &{{\hat{{\sigma}}_u}^{I^2}}=\frac{1}{n_I}||X^I_{u}||^2, & &{{\hat{{\sigma}}_{v|u}}^{I^2}}=\frac{1}{n_I}||X^I_{v}-\lambda_{u,v}^IX^I_{u}||^2.\label{eq.var.mle} 
\end{align}

   For the regression coefficients, since we know that they vary only across the datasets in which $v$ has been intervened upon, it is not possible to find a closed form expression for the estimator. To handle this, we tell apart the datasets in which $v$ has been intervened on from those in which it has not. Define $\lambda_{u,v}$ as the regression coefficient that is shared across the environments in which $v$ has not been intervened on. 
 Substituting \eqref{eq.var.mle} into \eqref{eq.loglik}, the part of the log-likelihood that depends on the regression coefficients becomes
 \begin{align*}
     -\big(\sum_{I\in\I_{\dot{v}}}\!\!\frac{n_I}{2}\log{||X^I_v\minus \lambda X^I_u||^2}+\!\!\sum_{I\in \I_v}\frac{n_I}{2}\log{||X^I_v\minus\lambda^I X^I_u||^2}\big).
 \end{align*}
  So if $I\in\I_v$ then $\hat{\lambda}^I=(X_u^I\cdot X_v^I)/||X_u^I||^2$, while $\hat{\lambda}$ is given by the solution of the 1-dimensional bounded optimization problem
  \begin{equation}
      \argmax_{\lambda\in[-1,1]} -\sum_{I\in\I_{\dot{v}}}\!\!\frac{n_I}{2}\log{||X^I_v\minus \lambda X^I_u||^2}.
  \end{equation}
For solving this problem, we find Brent's method, see e.g. \cite[Chap.~10.3]{press:2007}, to work well in practice. Finally, note that the dimension of the model, that is, the number of free parameters we need to optimize over, is  $2|\I|+1+|\I_{v}|$. Indeed, we have two variance coefficients for each interventional dataset, one regression coefficient for the datasets in which $v$ has not been intervened upon, and one regression coefficient for each dataset in which $v$ has been intervened on.
\begin{remark}
\label{rem.mar.variance}
\rm
In this section we allow the marginal variances $\sigma_u$ and $\sigma_{v|u}$ to vary also in the datasets in which $u$ and $v$ are not intervened upon. This is because the variance can change also as a consequence of an intervention on an ancestor. A precise likelihood computation would require to check for each experiment if there are possible ancestors in the graph that can affect the marginal variances, but this would be too computationally expensive while giving little benefit in terms of likelihood comparison.  We discuss this aspect further in the appendix, Prop.~A.4, where we also clarify that consistency is not affected by the relaxation of the variance constraints.\end{remark}

\subsubsection{Single Edge BIC Score After Collider Search}
\label{subsec.aftercoll.lik}
Let $G^{CP}=(V,E^{CP})$ be the CPDAG associated to $G$ and $G^U$ be graph obtained from $G^{CP}$ after removing all the oriented edges. We can then work on each connected component of $G^U$ independently. Let $U$ be one of these components, and $u-v$ be one of its $\I$-directly-identifiable edges. If we orient the edge from $u$ to $v$ then we would orient the remaining edges in $U$ in such a way that no other colliders appear.  We now propose a simple likelihood computation that takes this into account.

The choice of an orientation for the edges in $U$ reduces to the choice of a root vertex in it. If $u$ is a vertex in $U$, then we let $\mathcal{H}^U_u$ denote the model in which $u$ is the root. The likelihood function for this model factorizes in a simple way since each vertex has at most one parent.  We can write the log-likelihood of a datapoint as
\[
\log{f(X_U^I)}=\log{f(X_u^I)}+\sum_{v\neq u} \log{f(X_v^I|X^I_{pa(v)})},
\]
where each of the summands can be maximized independently from the others. This time the maximization is easier than the one in Section~\ref{subsec.likelihood.comp} because the variance parameters are shared across all environments in which a node has not been intervened upon. A closed form solution for this problem is provided in \cite{hauser:2015} and is given by $\hat{\sigma}_{v|\pa(v)}^2=\sum_{I_{\dot{v}}}\frac{n_I}{n_{\not v}}\hat{\sigma}_{v|\pa(v)}^{I^2}$, where $n_{\dot{v}}=\sum_{I_{\dot{v}}}n_I$ and $\hat{\sigma}_{v|\pa(v)}^{I^2}$ is the MLE of the conditional variance in the dataset $I$, while $\hat{\lambda}=\frac{\hat{\sigma}_{v,\pa(v)}}{\hat{\sigma}^2_{\pa(v)}}$.
The dimension of the model $\mathcal{H}_u$ is $1+|\I_u|+2\sum_{v\neq u}(1+|\I_{v}|)$.
 
 \subsection{Collider Search}
\label{sec.collidersearch}

\subsubsection{BIC Score for Collider Search}
\label{subsec.vstruc.before}
If the skeleton has a triplet $u-v-w$, we can decide whether it forms a collider or not by testing for the independence of $u$ and $w$ in all the environments. Log-likelihood ratio statistics for a test are provided by function 1 defined in Section~\eqref{eq.weights}, applied to $x_I=\rho^I_{u,w}$.
Here the difference in the dimension between the model $\mathcal{H}_{u\to v\leftarrow w}$ and the other 3 models in which $u$ and $w$ are not independent is $|\I|$.  

\subsubsection{BIC Score for Collider Completion}
\label{subsec.vstruc.after}
Even though testing for the independence of $u$ and $w$ as in Section~\ref{subsec.vstruc.before} would give a consistent procedure also in the case in which we have already oriented $u\to v$, in this case a more refined analysis of the likelihood is possible. Indeed, here we can test the model $u\to v\to w$ against $u\to v\leftarrow w$. In the former case the likelihood factorizes as $f(X_u^I)f(X_v^I|X_u^I)f(X_w^I|X_v^I)$ and  can be maximized as in Section~\ref{subsec.likelihood.comp}. In the latter case it factorizes as $f(X_u^I)f(X_w^I)f(X_v^I|X_u^I,X_w^I)$,  and the MLE for $\lambda_v^I=(\lambda^I_{v|u},\lambda^I_{v|w})$ in the environments in which $v$ has been intervened upon is given by the usual estimator $\hat{\Sigma}_{u,w}^{I^{-1}}\hat{\Sigma}_{v|u,w}^I$, while the common regression coefficients for the environments in which $v$ hasn't been intervened on  $\hat{\lambda}_{v\mid u},\hat{\lambda}_{v\mid u}$, is given by the solution of the following 2 dimensional optimization problem:
\begin{align*}
   \argmax_{\lambda_{v\mid u},\lambda_{v\mid u}\in[0,1]^2}\sum_{I\in\I_{\dot{v}}}\!-\!\frac{n_I}{2}\log{||X^I_v\minus \lambda_{v\mid u} X^I_u\minus \lambda_{v\mid w} X^I_w||^2}
\end{align*}
The dimension of the model $\mathcal{H}_{u\to v\to w}$ is $3|\I|+|\I_v|+|\I_w|+2$ while that of the model $\mathcal{H}_{u\to v\to w}$ is $2(2|\I|+|\I_v|+1)$.

\subsection{Complete Orientation Procedures}\label{subsec.orient.proc}
\noindent
We propose two  procedures for the orientation. The pseudocode for each, including subroutines, is in
Appendix \ref{app.sec.pseudocode}: 
\begin{enumerate}
    \item[P.1] Compute the CPDAG using the collider search, then use single edge orientation.
    \item[P.2] Use single edge orientation first, then orient the rest of the $\I$-CPDAG using the collider search.
\end{enumerate}
These procedures differ in the amount of statistical vs.~causal information they extrapolate from data, with P.1 being the more statistically oriented, while P.2 the more causally oriented. 
Note that for the single edge orientation we can choose between BIC (Sec.~\ref{subsec.likelihood.comp}) and IRC (Sec.~\ref{subsec.common.reg.coeff}), and  to find colliders we can use the more general setting of Section~\ref{subsec.vstruc.before} hereon referred to as ``simple'', or the more refined analysis of Section~\ref{subsec.vstruc.after} hereon referred to as ``refined''. 

Whenever we use the orientation procedures involving a BIC score, it is intended that we solve a local model selection problem in the following way: we compute the MLE estimator for the possible models (e.g., $\mathcal{H}_{u \to v}$ and $\mathcal{H}_{v \to u}$ for the single edge orientation) and then plug it into the likelihood function to compute the maximum likelihood $\hat{L}$, the BIC score of the model $\mathcal{H}$ will be $\log(n)\dim(\mathcal{H})-2\log(\hat{L})$, where $n$ is the sample size. Finally, we select the model with the highest BIC score.

\section{Simulation Studies}
\label{sec.sim.study}
In this section we first assess the performance of the
different skeleton and orientation recovery procedures, then 
we construct full versions
of our algorithms to compare to other methods.
The setup for our simulations is explained in Appendix \ref{app.descr.comp}, the code is available at \cite{github}. 

\subsection{Skeleton and Orientation Recovery}
\noindent
 Fig.~\ref{fig.skeleton_trends} shows the structural Hamming distance (SHD) in skeleton recovery for the three aggregation functions introduced in Section~\ref{sec.skeleton}.  The SHD is the minimum number of edge additions, deletions and reversals necessary to transform one graph into another. We see that the SHD decreases for increasing number of samples, with little effect of the number of datasets and the size of the intervention. Fig.~\ref{fig.skel_runtime} in Appendix \ref{app.sec.orientation}. shows that the weighted mean is the  fastest. 
  A more detailed analysis is contained in Appendix \ref{app.subsec.skel.rec}.  

  \begin{figure}[ht]
     \centering
     \includegraphics[scale=0.4]{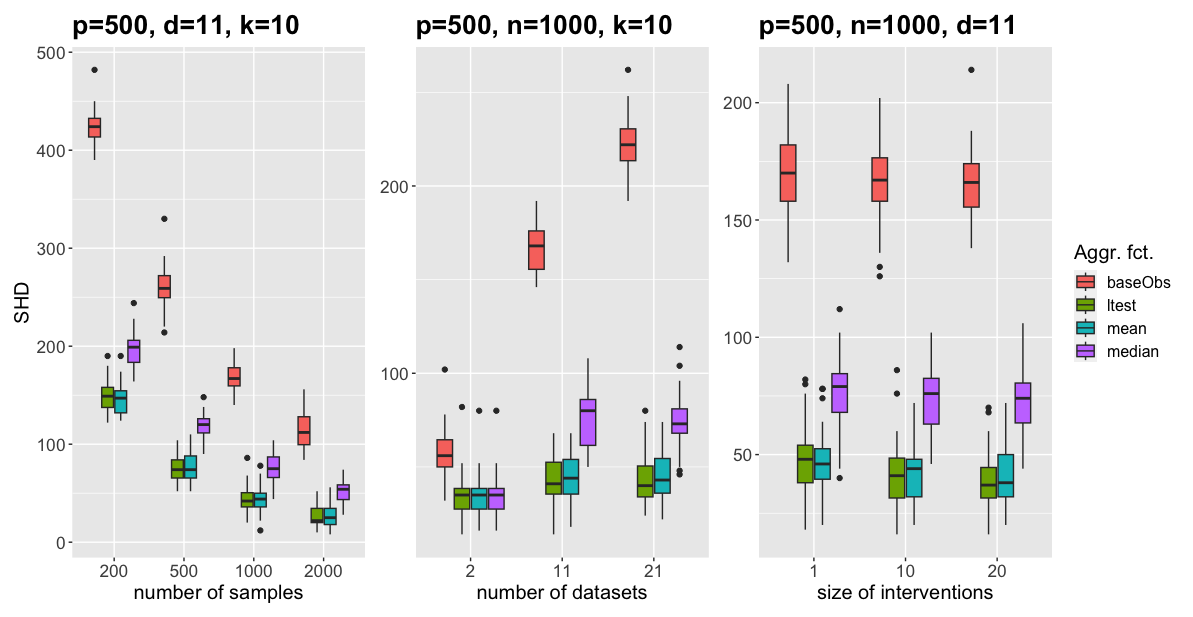}
     \caption{Skeleton recovery  with $p=500$ nodes,   $N=200,500,1000,2000$ samples, $d=2,11,21$ datasets (1 interventional and $d-1$ observational) and $k=1,10,20$ nodes targeted per intervention, respectively. 
    The aggregation function baseObs denotes a baseline with observational data only. The labels Itest, mean, median indicate each of the procedures in the list in Section~\ref{sec.skeleton} respectively.}
     \label{fig.skeleton_trends}
 \end{figure}
 Fig.~\ref{fig.orientation_trends} depicts the SHD in orientation recovery, conditional on the skeleton being correct.
 Interestingly, the more refined test proposed in Section~\ref{subsec.vstruc.after} gives no benefits compared to the ones in Section~\ref{subsec.vstruc.before}, and it is  slower to compute. P.1 also performs better than P.2 in general. 
 More details on the orientation procedure  in Appendix \ref{app.sec.orientation}.
  \begin{remark} \label{rem:fixed:d}
  \rm
     Notice that in the experiment shown in the middle panels of Figs.~\ref{fig.skeleton_trends} and \ref{fig.orientation_trends}, the overall sample size is fixed, so increasing $d$ significantly reduces the sample size in each one of the datasets, resulting in a much more difficult learning problem. This explains why the orientation procedures seem to be negatively affected by the introduction of interventional datasets.
 \end{remark}
   \begin{figure}[ht]
     \centering
     \includegraphics[scale=.4]{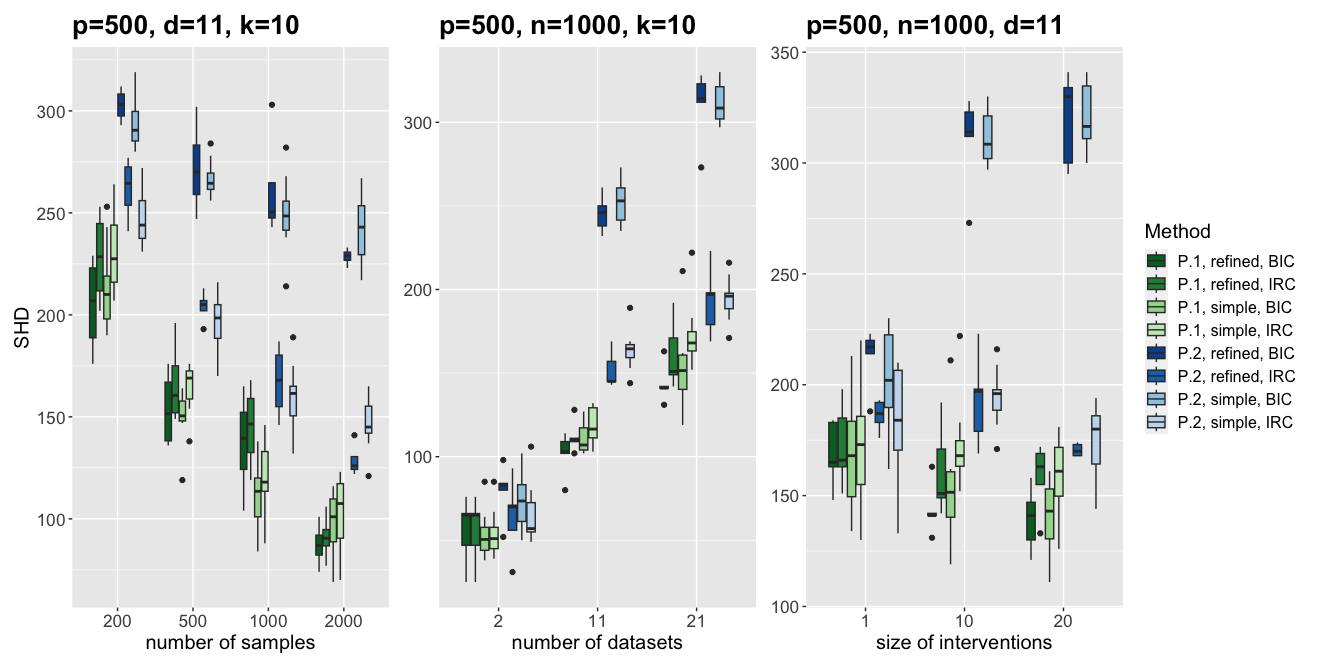}
     \caption{Orientation recovery with base parameters 500 nodes, 1000 samples and 10  intervention targets with 10 intervened nodes each. For all proposed methods we see the convergence when the sample size increases. In this case increasing the number of datasets negatively affects the performance, this point is disccused in Rem.~\ref{rem:fixed:d}.}
     \label{fig.orientation_trends}
 \end{figure}
 
\label{subsec.sim.skeleton}
\subsection{Complete Algorithm}
\noindent
Although our focus is on causal discovery in  high-dimensional settings, for completeness,  we compare
the performance of our full skeleton and orientation recovery 
procedures (Sec.~\ref{subsec.orient.proc}) to the performance of GIES \cite{Hauser:20212}, DCDI \cite{brouillard:2020}, BaCaDI \cite{hagele:2022} and IGSP \cite{wang:2017} in low dimension with simulations on DAGs with $p=20$ nodes. The results are summarized in Fig.~\ref{fig.low.dim} (left). We see that general algorithms achieve a better accuracy, but this comes with a high price in terms of running time that makes them infeasible for larger graphs. We refer to Table \ref{app.tab.runtime} 
for a runtime analysis.

For high-dimensional settings, to leverage accuracy and fast computation,
we construct our full algorithm using either P.1 or P.2 with mean as 
aggregation function, simple for collider search and IRC for single edge
orientations. Among the four algorithms from the literature considered in
Fig.~\ref{fig.low.dim} (left), the only one that would terminate in less than 24h for DAGs with $p=500$ nodes is GIES \cite{hauser:2015}. Thus, we
compare our two high-dimensional versions of the algorithm with GIES only.
The Fig.~\ref{fig.low.dim} (right) shows the results of the simulations. We see in this setting that our algorithm provides better accuracy in terms of structural Hamming distance in addition to being considerably faster, as the computation times in Table~\ref{tab.runtime_high.dim} indicate.  Appendix \ref{app.subsec.furth.comp} contains an extensive comparison of our algorithm against GIES.

  \begin{figure}[ht]
     \centering
     \includegraphics[scale=.4]{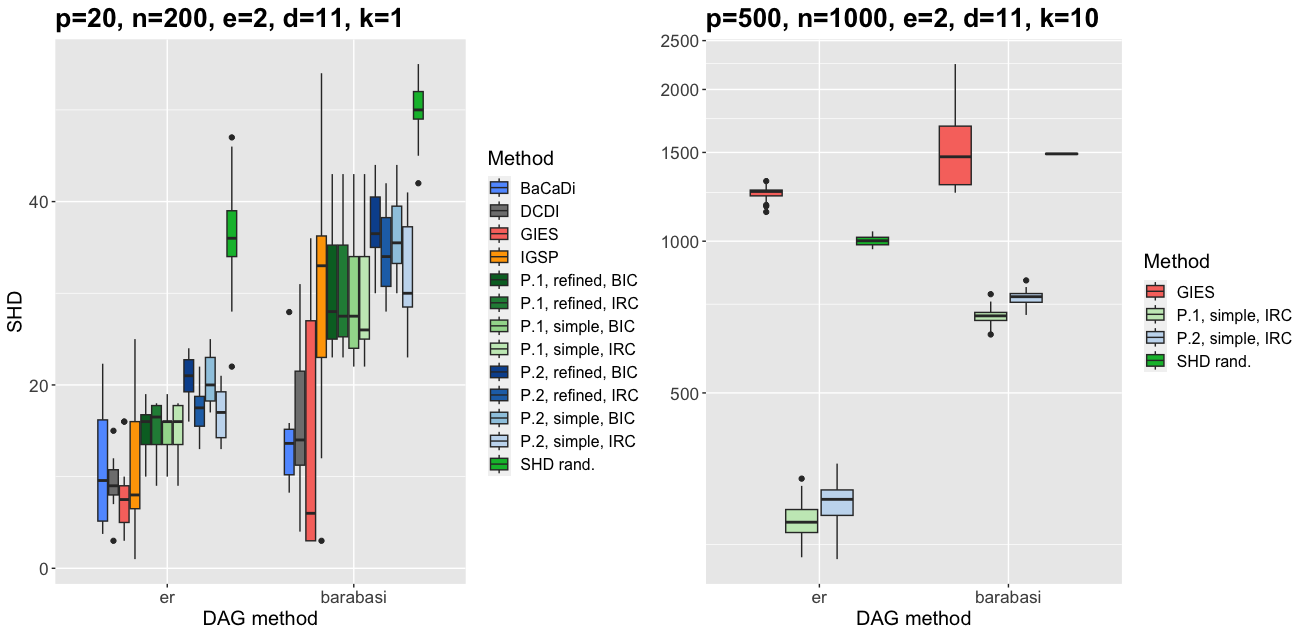}
     \caption{Performance of our algorithm against different baselines in low- and high-dimensional settings on DAGs for random Erd\H{o}s-R\'{e}nyi or Barabasi-Albert graphs with $p=20, 500$ nodes and expected number of edges per node $e=2$. 
    The observational sample sizes are respectively $100$ and $48$, the other samples are evenly distributed in the interventional datasets. Notice that the $y$-axis for the plot on the right is in log-scale.}
     \label{fig.low.dim}
 \end{figure}

To emphasize the point that our algorithms are able
to retain structural information of the original DAG, in
 both the plots in Fig.~\ref{fig.low.dim} we added ``SHD rand.'' as baseline. This shows the SHD between the $\I$-CPDAG of the true DAG and the $\I$-CPDAG of randomly generated polytree. This baseline makes sure to highlight that the better performance of our algorithms in terms of SHD is not given by the sparsity structure of the polytree. 
 The distance from the result of our algorithm to this baseline is a strong indication that even in a misspecified setting where the generating graph is not a polytree, our algorithms are still able to infer information about the original DAG.
\begin{table}[ht]
    \centering
        \caption{Mean runtime on  DAGs with 500 nodes and different number of samples $n$ corresponding to
    Fig.~\ref{fig.low.dim}. 
    }
        \begin{tabular}{l|r|r|r|r|r|r|}
        & \multicolumn{6}{c|}{Runtime in seconds} \\
        \!Method\! &  \multicolumn{2}{c|}{$n=500$}  &  \multicolumn{2}{c|}{$n=1000$} &  \multicolumn{2}{c|}{$n=2000$} \\
         & mean & max & mean & max &mean & max \\
        \hline
        \!P.1  & 8 & 12 & 8 & 13 & 8 & 13 \\
        \!P.2  & 8& 14 & 7.8 & 13 & 8 & 15 \\
        \!GIES &7960 & 14715 & 449 & 909 & 284 & 553 \\
    \end{tabular}

    \label{tab.runtime_high.dim}
\end{table}

\subsection{Protein Expression Data}
\label{sec.gene.exp}
\noindent
We illustrate our algorithm on the well-known protein expression dataset from \cite{sachs:2005}. Similar to \cite{wang:2017}, we processed the dataset into one observational and five one-node interventional datasets with 5846 samples of 11 nodes in total. The CPDAG of the conventionally accepted model of the interactions between nodes serves as ground truth;  this consensus $\I$-CPDAG and the one estimated by our algorithm can be found in the Appendix. Table~\ref{tab.sachs} shows the SHD to the consensus $\I$-CPDAG; the statistics for the other algorithms are taken from \citep[Tab.~3]{brouillard:2020}. Although the ground truth is a DAG with 18 edges, our algorithm's performance is comparable to that of more general methods.

\begin{table}[ht]
    \centering
        \caption{Structural Hamming Distance to the consensus $\I$-CPDAG of the protein network from \cite{sachs:2005}.}
    \begin{tabular}{l|l|l|l|l|l|l} 
    \multicolumn{6}{c}{\hspace{2cm}Method}\\ \hline
          &P.1/P.2 & GIES & IGSP & CAM & DCDI-G &DCDI-DSF  \\ 
         \hline
           SHD &15 & 38 & 18 & 35 & 36 & 33
    \end{tabular}
    \label{tab.sachs}
\end{table}

\section{Conclusion}
\label{sec.concl}
We proposed an approach to learn linear Gaussian polytree models from interventional data where the interventions have known targets. Our two-step approach of first learning the skeleton and then orienting the edges exploits the availability of interventional data at each step, all the while using only local computations with low-dimensional marginals.
We emphasize that our methods are especially well suited for the high-dimensional setting, 
with large graphs but moderate sample size.

To conclude, we highlight topics that emerge as future directions.

{\em Unknown interventions}: 
A natural extension of our work is to allow for unknown intervention targets.
Although the skeleton recovery and collider search procedures apply in this case, 
the remaining orientation subroutines do not because they operate by searching $\I$-identifiable edges; this, in turn, depends on the characterization
of $\I$-Markov equivalence which requires
explicit knowledge of $\I$. The notion of $\Psi$-Markov equivalence, introduced in \citet{jaber:2020}, for 
unknown intervention targets appears to be a reasonable candidate to replace 
the use  of the $\I$-MEC. The reason being that  the $\Psi$-Markov equivalence
class ($\Psi$-Markov EC) is a generalization of the $\I$-MEC in this case \citep[App. D]{jaber:2020}.

{\em Forests}:  We focused on 
connected polytrees, but 
a generalization to disconnected forests of polytrees would be of interest. Undirected forests have been 
studied by \citet{edwards:2010}, but their construction does not carry over naturally to the directed case, and new research is needed.

{\em Hidden variables:}
The algorithms we propose do not address the case where some variables remain hidden.
For purely observational data from a polytree model, 
the hidden variable setting was considered by 
\citet{sepher:2019}.
These authors provide necessary and sufficient conditions for causal structure 
recovery of the polytree with hidden nodes. 
It would be interesting to investigate to what extent
interventional data would improve the identifiability of the polytree structure when such conditions are not met. Similar to the unknown interventions case,
a possible approach would be to replace the $\I$-MEC by the $\Psi$-Markov EC because its level of generality encompasses  hidden variables also.

\section*{Acknowledgements}
This project has received funding from the European Research Council (ERC) under the European Union’s Horizon 2020 research and innovation programme (grant agreement No 883818). DT's PhD scholarship is funded by the IGSSE/TUM-GS via a Technical University of Munich--Imperial College London Joint Academy of Doctoral Studies. ED was supported by the FCT grant 2020.01933.CEECIND, and partially supported by CMUP under the FCT grant UIDB/00144/2020.

\appendix

\counterwithin{figure}{section}
\counterwithin{table}{section}

\section{Notes on Consistency}
\label{app.sec.cons}
 Notation and assumptions: 
\begin{enumerate}[label=\alph*)]
\label{app.assumptions}
    \item Let $\mathcal{G}=(V,E)$, with $|V|=p$, be the polytree that underlies the data-generating distribution.  Let $\mathcal{S}$ be its skeleton,  and let $\I=\{I_1,\dots,I_k\}$ be a conservative set of intervention with $\Sigma_I$ being the covariance matrix associated to  interventional setting $I\in\I$.  We write $\rho_{u,v}^I$ for the correlation between $X_u$ and $X_v$ in the setting $I$.  Finally, let $w=(w_1,\dots,w_k)\in[0,1]^k$, $\sum_j w_j=1$, be the vector giving the probabilities that the $i$-th sample is taken from the interventional setting $I_j$.
    \item Following \cite{hauser:2015}, we assume that we jointly sample the interventional setting $I_i$ and $X^{I_i}_i$, and that 
    \begin{equation*}
    X^{I_i}_i|I_i\sim \mathcal{N}(0,\Sigma_{I_i}).
    \end{equation*}
    \item Let $X=(X^{I_1}_{1},X^{I_2}_2,\dots,X^{I_n}_{n})\in\mathbb{R}^{p\times n}$ be our sample.
    We write $\rho_{u,v}^I$ and $\hat{\rho}_{u,v}^I$ for the population and the sample correlation between $X^I_u$ and $X^I_v$.  Here, the sample quantity is computed using $X^I=(X^I_{i_1},\dots,X^I_{i_{n_I}})$, where $\{i_1,\dots,i_{n_I}\}$ is the set of all indices $j$ with $I_{i_j}=I$.
    \item Let $\rho^I_{m}:=\min_{u,v}|\rho^I_{u,v}|$ and $\rho^I_{M}:=\max_{u,v}|\rho^I_{u,v}|$, where the min and max extend over pairs of nodes $u,v$ that are joined by an edge in the polytree $\mathcal{G}$.  We assume 
    \[
    0<\rho^I_{m}<\rho^I_{M}<1 \ \text{ for all } I\in\I.
    \]
    \item Let $\gamma_I=\rho^I_{m}(1-\rho^I_M)/2$.
    \label{item:bound}
\end{enumerate}

\subsection{Consistency of Skeleton Recovery}
\label{app.consistency.skeleton}
\noindent
In this section we provide finite sample analysis  and prove consistency for the skeleton recovery under the assumptions a)-e) laid above.
The following lemma is a straightforward modification of \citet[Thm.~3.4]{lou:2021}.
\begin{lemma}
\label{app.lem.cl.fail}
If for all $u,v\in V$ and all $I\in\I$ we have 
\[
|\hat{\rho}^I_{u,v}-\rho^I_{u,v}|\leq\gamma_I,
\]
then the Chow-Liu algorithm computed with the weighted mean function outputs the correct skeleton of the tree.
\end{lemma}

\begin{proposition}
\label{prop.consistency}
 Let $\hat{\mathcal{S}}_n$ be the output of the Chow-Liu algorithm computed with the weighted mean function, on a sample of size $n$, then \begin{align*}
        &\mathbb{P}(\hat{\mathcal{S}_n}\neq\mathcal{S})\leq C_1\binom{p}{2}(n-2)\bigg[\exp(-n)+\sum_{I\in\I}\exp{\bigg[-\bigg(1-(4-\frac{n}{\beta_I})/nw_I\bigg)^2\frac{nw_I}{2}\bigg]}\bigg]\\
        &+C_1\binom{p}{2}(n-2)\bigg[\sum_{I\in\I}\exp{\bigg[-\bigg(1-4/nw_I\bigg)^2\frac{nw_I}{2}\bigg]}\bigg], 
 \end{align*}
 where $C_1$ is a positive constant and $\beta_I = \log\frac{4-\gamma_{I}^2}{4+\gamma_{I}^2}<0$.
\end{proposition}
\begin{proof}
From Lemma~\ref{app.lem.cl.fail}, we have  
\begin{align*}
    \mathbb{P}(\hat{\mathcal{S}_n}=\mathcal{S})\ge \mathbb{P}(\displaystyle\bigcap_{\substack{u,v\in V\\ I\in\I}} \{|\hat{\rho^I_{u,v}}-\rho^I_{u,v}|\leq\gamma_I\})
    =1-\mathbb{P}(\bigcup_{\substack{u,v\in V\\ I\in\I}}\{|\hat{\rho}^I_{u,v}-\rho^I_{u,v}|>\gamma_I\}).
\end{align*}
Now, we prove that this lower bound converges to 1 as $n\to\infty$ by showing that \begin{equation*}\mathbb{P}(\bigcup_{\substack{u,v\in V\\ I\in\I}}\{|\hat{\rho}^I_{u,v}-\rho^I_{u,v}|>\gamma_I\})\xrightarrow[]{n\to \infty} 0.\end{equation*} 
Indeed, by conditioning on the size $(n_1,\dots,n_k)$ of the interventional data sets, we can write
\begin{align*}
\mathbb{P}(\bigcup_{\substack{u,v\in V\\ I\in\I}}\{|\hat{\rho}^I_{u,v}-\rho^I_{u,v}|>\gamma_I\})=\sum_{\mathbf{n}=(n_1,\dots,n_k)}\mathbb{P}\bigg(\bigcup_{\substack{u,v\in V\\ I\in\I}}\{|\hat{\rho}^I_{u,v}-\rho^I_{u,v}|>\gamma_I\}\mid \mathbf{n}\bigg)\mathbb{P}(\mathbf{n}).
\end{align*}
Applying a union bound, we get
\begin{multline*}
\mathbb{P}(\bigcup_{\substack{u,v\in V\\ I\in\I}}\{|\hat{\rho}^I_{u,v}-\rho^I_{u,v}|>\gamma_I\})
\le\sum_{u,v\in V, u\not=v}\sum_{\mathbf{n}=(n_1,\dots,n_k)}
    \mathbb{P}\bigg(\bigcup_{I\in\I}\{|\hat{\rho}^I_{u,v}-\rho^I_{u,v}|>\gamma_I\}\mid \mathbf{n}\bigg)\mathbb{P}(\mathbf{n}).
\end{multline*}
Using \citet[Lemma 1]{kalish:2007}, the conditional probabilities can be upper bounded as
\begin{align*}
&\sum_{\mathbf{n}=(n_1,\dots,n_k)}\mathbb{P}\bigg(\bigcup_{I\in\I}\{|\hat{\rho}^I_{u,v}-\rho^I_{u,v}|>\gamma_I\}\mid \mathbf{n}\bigg)\mathbb{P}(\mathbf{n}) \\
&\le 
    C\sum_{\substack{\mathbf{n}=(n_1,\dots,n_k)\\n_I>4, \forall I\in\I}}\sum_{I\in\I}(n_I-2)\exp{\bigg((n_I-4)\beta_I\bigg)}\mathbb{P}(\mathbf{n})
    +c\mathbb{P}(\exists I\mid n_I\leq4)\\
   & \le C(n-2)\sum_{\substack{\mathbf{n}=(n_1,\dots,n_k)\\n_I>4, \forall I\in\I}}\sum_{I\in\I}\exp{\bigg((n_I-4)\beta_I\bigg)}\mathbb{P}(\mathbf{n})+c\mathbb{P}(\exists I\mid n_I\leq4)
\end{align*}
for some positive constants $c,C$ and the negative coefficient 
\[
\beta_I=\log{\frac{4-\gamma_I^2}{4+\gamma_I^2}}.
\]
For all $\delta>0$, define 
\[
A_\delta=\{(n_1,\dots,n_k): \exp((n_I-4)\beta_I)<\delta,\forall I\in\I\},
\]
and let $B_\delta=A_\delta^c$ be its complement.  Then we may bound our probability of interest as
\begin{align*}
\mathbb{P}(\bigcup_{\substack{u,v\in V\\ I\in\I}}\{|\hat{\rho}^I_{u,v}-\rho^I_{u,v}|>\gamma_I\}) 
\le \binom{p}{2}(n-2)\bigg[\sum_{\substack{(n_1,\dots,n_k)\\n_I>4, \forall I\in\I}}\sum_{I\in\I}\exp{\bigg((n_I-4)\beta_I\bigg)}\mathbb{P}(\mathbf{n})+c\mathbb{P}(\exists I\mid n_I\leq4)\bigg]\end{align*} that is upper bounded by
\begin{align*}
    &C_1\binom{p}{2}(n-2)\bigg[(\delta\mathbb{P}(A_\delta)+\mathbb{P}(B_\delta))+c\mathbb{P}(\exists I\mid n_I\leq4)\bigg]\le C_1\binom{p}{2}(n-2)\bigg[(\delta+\mathbb{P}(B_\delta))+c\mathbb{P}(\exists I\mid n_I\leq4)\bigg]
\end{align*}
for some positive constant $C_1$, where in the last inequality we only used $\mathbb{P}(A_\delta)\leq1$.

In order to do so, we rewrite 
\begin{align*}
  B_{\delta}&=\{(n_1,\dots,n_k)\mid\exists I\in\I ,\,n_I<4+\log(\delta)/\beta_I\}\\
  &=\bigcup_{I\in\I}\{n_I<4+\log(\delta)/\beta_I\}.  
\end{align*}
Since $n_I\sim \text{Bin}(n,w_I)$, using Chernoff's bound we find that 
\begin{equation}
    \label{eq:chernoff}
    \mathbb{P}(B_\delta)\le \sum_{I\in\I}\exp{\bigg[-\bigg(1-(4+\frac{\log(\delta)}{\beta_I})/nw_I\bigg)^2\frac{nw_I}{2}\bigg]}.    
\end{equation}
Using $\delta=\exp(-n)$ we get the following bound:
\begin{equation}
    \begin{aligned}
        \label{eq:bound}
    &\mathbb{P}(\hat{\mathcal{S}_n}\neq\mathcal{S})\leq
     C_1\binom{p}{2}(n-2)\bigg[\exp(-n)
    +\sum_{I\in\I}\exp{\bigg[-\bigg(1-(4-\frac{n}{\beta_I})/nw_I\bigg)^2\frac{nw_I}{2}\bigg]}\bigg]\\
    &+C_1\binom{p}{2}(n-2)\sum_{I\in\I}\exp{\bigg[-\bigg(1-4/nw_I\bigg)^2\frac{nw_I}{2}\bigg]},        
    \end{aligned}
\end{equation}
where we bounded $\mathbb{P}(\exists I\mid n_I\leq4)$ with the same Chernoff bound we used in Eq.~\eqref{eq:chernoff}.
\end{proof}

\begin{proposition}
Assume that the number of experiments $|\mathcal{I}|$ does not depend on $n$, and that there are constants $\gamma_*,w_*>0$ such that $\min_{I\in\mathcal{I}} \gamma_I>\gamma_*$ and $\min_{I\in\mathcal{I}} w_I>w_*$ uniformly in $n$. If $\log p = o(n)$, then
\begin{equation*}
    \mathbb{P}(\hat{\mathcal{S}_n}\neq\mathcal{S})\xrightarrow{n\to\infty} 0.
\end{equation*}
\end{proposition}
\begin{proof}
    Let $\beta_*=\log{\frac{4-\gamma_*^2}{4+\gamma_*^2}}<0$.  Then $\beta_I<\beta_*<0$ for all $I\in\mathcal{I}$.  Thus, there is a constant $C'\equiv C'(\beta_*,w_*)>0$ such that for all large $n$, all terms on the right-hand side of Eq.~\eqref{eq:bound} are bounded from above by 
    \[
    C_1p^2n\exp{\bigg[-C'\frac{nw_*}{2}\bigg]} < C_1p^2\exp{\bigg[-C_2n \bigg]}
    \]
    for a further smaller constant $C_2=C_2(\beta_*,w_*)$.  Hence the consistency claim holds under the assumption that $\log p=o(n)$.
\end{proof}

\subsection{Consistency of Orientation Procedures}
\noindent
A finite-sample analysis of the orientation procedures is more subtle, and what we present below is instead a discussion of consistency in the setting of growing sample size but fixed dimension.  The difficulties in a more refined analysis lie primarily with 
the BIC-based procedure as it involves tests based on likelihood ratios that are obtained by optimizing a generally nonconvex function.  For the IRC procedure, a high-dimensional analysis could be performed using concentration inequalities for non-central chi-square distributions \citep{ghosh:2021}.

\subsubsection{BIC}

In \cite{hauser:2015}, it is  proven that under the  assumptions a)-e) stated at the beginning of this section of the supplement the model can be parametrized as an exponential family, so the consistency for the orientation part follows from classical consistency results on the BIC criterion for exponential families, see e.g.~\citet[Thm.~5.13]{drton:2009}.
\begin{proposition}
    Let $u\to v\in G$ be directly $G$-identifiable, and let $\mathcal{H}_{u\to v}, \mathcal{H}_{u\leftarrow v}$ be the correctly specified models in which we allow the variance to vary only if there is a possible ancestor of the node that has been intervened upon, while  $\Tilde{\mathcal{H}}_{u\to v}$  and $\Tilde{\mathcal{H}}_{u\leftarrow v}$ are the models in which we relax the variance constraints and let them vary also in the environments in which there are no ancestors that have been intervened upon. Then under generic data-generating distributions in ${\mathcal{H}}_{u\to v}$ we have:
    \begin{equation*}
        \mathbb{P}(l_{\Tilde{\mathcal{H}}_{u\to v}}>l_{\Tilde{\mathcal{H}}_{u\leftarrow v}})\xrightarrow[]{n\to\infty}0.
    \end{equation*}
\end{proposition}
\begin{proof}
    Consider first the case in which there is an $I_0\in\I$ such that $I_0\cap\{u,v\}=\{u\}$. 
    In this case, it holds that $\mathbb{P}(l_{\Tilde{\mathcal{H}}_{u\to v}}>l_{\Tilde{\mathcal{H}}_{u\leftarrow v}})\xrightarrow[]{n\to\infty}0$ provided the true data-generating distribution belongs to $\mathcal{H}_{u\to v}$ but not the relaxed alternative $\Tilde{\mathcal{H}}_{v\to u}$.  We now claim that the intersection $\mathcal{H}_{u\to v}\cap\Tilde{\mathcal{H}}_{v\to u}$  corresponds to a measure zero subset of $\mathcal{H}_{u\to v}$.  To see this note that even in the larger model $\Tilde{\mathcal{H}}_{v\to u}$ we have $\lambda^{\emptyset}_{v\to u} = \lambda^{I_0}_{v\to u}$, where the involved quantities can be written as rational functions of the model parameters:
    \begin{equation*}
        \lambda^{I}_{v\to u}=\frac{\lambda^I_{u\to v}\sigma^I_u}{(\lambda^I_{u\to v})^2\sigma^I_u+\sigma^I_{v\mid u}},\qquad I={\emptyset, I_0}.
    \end{equation*}
    Hence, imposing the condition $\lambda^{\emptyset}_{v\to u} = \lambda^{I_0}_{v\to u}$ amounts to a \emph{non-trivial} rational equation also on $\mathcal{H}_{u\to v}$, which implies that $\mathcal{H}_{u\to v}\cap\Tilde{\mathcal{H}}_{v\to u}$ is indeed a measure zero set in $\mathcal{H}_{u\to v}$.

    We turn to the second case, where there is no set $I_0\in\I$ with $I_0\cap\{u,v\}=\{u\}$.  Since we assume the orientation of the edge $\{u,v\}$ to be $\I$-directly-identifiable,
    there has to be at least one $I_1$ such that $I_1\cap\{u,v\}=\{v\}$.  This implies that the dimension of $\Tilde{\mathcal{H}}_{u\to v}$ is smaller than that of $\Tilde{\mathcal{H}}_{v\to u}$, so in the event in which the true parameters lie in $\mathcal{H}_{u\to v}\cap\Tilde{\mathcal{H}}_{v\to u}\subseteq \Tilde{\mathcal{H}}_{u\to v}\cap~\Tilde{\mathcal{H}}_{v\to u}$, the BIC will asymptotically select the model of lower dimension with probability converging to 1, which we have proven to be $ \Tilde{\mathcal{H}}_{u\rightarrow v}$.
    \end{proof}
\subsubsection{Invariance of Regression Coefficients (IRC)}
Orientating  edges using the IRC treats each $\I$ 
directly identifiable edge individually. Let $n=n_{I_0}+n_{I}$ be the total
sample size to perform the test in Section~IV-A of the main text. 
To justify the consistency of this test, note that as $n\to \infty $ it is possible to 
choose a sequence of significance levels $\alpha_n$ that decreases to zero suitably slowly, in such a way that the probability of a type I or a type II error both tend to zero. In practice, we treat the significance level as a tuning parameter of the algorithm.

\section{Orientation Procedures: Example and Pseudocode}
\label{app.sec.pseudocode}

The next example demonstrates  each one of the procedures for a DAG on five nodes and a collection of intervention targets. Following the example,  we provide pseudocode to perform each of the procedures outlined in Section~IV.C and Section~IV.D
of the main text.
These procedures rely on several subroutines, these are presented first in
 Algorithms \ref{alg.rec.vstruct},\ref{app.alg.triplet}, and \ref{alg.fndroot}. 
Immediately after we present the pseudocode for the procedures P. 1, P. 2 
in Algorithms \ref{alg.proc.1}, \ref{alg.proc.2}. 

\begin{example} \label{ex.orientall}
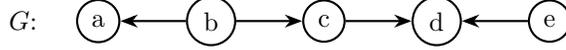
\begin{figure}[!t]
\centering
\begin{tikzpicture}[scale =0.5]
    \node (a) at (-8,0) {$G$:};
\begin{scope}[every node/.style={circle,thick,draw}]

    \node (a) at (-6,0) {a};
    \node (b) at (-3,0) {b};
    \node (c) at (0,0) {c};
    \node (d) at (3,0) {d};
    \node (e) at (6,0) {e};
\end{scope}

\begin{scope}[>={Stealth[black]},
              every edge/.style={draw=black,thick}]
    \path [->] (b) edge node{} (a);
    \path [->] (b) edge node{} (c);
    \path [->] (c) edge node{} (d);
    \path [->] (e) edge node{} (d);
\end{scope}
\end{tikzpicture}
\caption{Example~\ref{ex.orientall} describes the procedures P. 1 and P. 2 to learn the DAG $G$ above with $\I=\{I_1=\emptyset,I_2=\{b\}\}$. The combination $(G,\I)$ is such that every edge in $G$ is $\I$-directly-identifiable.
}\label{app.fig.graph}
\end{figure}

Here we show how to learn the graph in the Fig.~\ref{app.fig.graph} where $\I=\{I_1=\emptyset,I_2=\{b\}\}$ and the skeleton of the graph is already known.
\begin{enumerate}
    \item[(i)] Procedure 1 would first test $X_a\indep X_c,X_b\indep X_d $ and $X_c\indep X_e$, finding that only the last triplet forms a collider and so orienting $c\rightarrow d\leftarrow e$. After this step, to orient the subtree $a-b-c$  Procedure 1 would either compute the BIC scores of the three models $\mathcal{M}_a=a\rightarrow b\rightarrow c$, $\mathcal{M}_b=a\leftarrow b\rightarrow c$ and $\mathcal{M}_c=a\leftarrow b\leftarrow c$, finding that $b$ is the root, or would use the test for the invariance of the regression coefficients for $a-b$ and $b-c$ orienting the two edges independently.
    \item[(ii)] Procedure 2 would orient the edges $a\leftarrow b$ and $b\rightarrow c$ using using either the BIC score or the IRC test. Then, it would orient the edge $c\rightarrow d$ testing $X_b\indep X_d$ and finally orient $d\leftarrow e$ testing $X_c\indep X_e$.  
\end{enumerate}
\end{example}

\begin{algorithm}[!t]\caption{FINDTHEROOT$(\I,X,N,\mathcal{G})$}
\KwData{A set of intervention targets $\I$, a list of datasets $X$,
a list of sample sizes $N$, an undirected graph $\mathcal{G}$.}
\KwResult{A directed tree $\mathcal{G}_d$}
\SetKwFunction{FNDROOT}{FINDTHEROOT}
\label{alg.fndroot}
$\mathcal{G}=(V,E)$\\
$L\leftarrow0$ vector of length $|V|$\\
\For{$v\in V$}{
$L(v)\leftarrow $Likelihood of the model in which $v$ is the root
}
$r\leftarrow\argmax L$\\
$\mathcal{G}_d\leftarrow$ directed tree with skeleton $\mathcal{G}$ and root node $r$
\end{algorithm}

\begin{algorithm}[!t]\caption{COLLIDER$(I,X,N,E,O,COLL)$}
\label{app.alg.triplet} 
\SetKwFunction{Triplet}{COLLIDER}
\KwData{A set of intervention targets $\I$, a list of datasets $X$, 
a list of sample sizes $N$, a list of unoriented edges $E$, a list of oriented edges $O$, an indicator $COLL$ indicates the BIC score to use for the collider search when one edge has already been oriented.}
\KwResult{An augmented list of oriented edges $O$, a list of unoriented edges $E$}
\label{alg.samp.pair_trip}
    \While{$F==$TRUE}{
        $F\leftarrow$TRUE\\
        \For{$u\to v\in O$}{
            \For{$v-w\in E$}{
                \eIf{$X^I_u\indep X^I_w,\forall I\in\I$\tcc*{using $COLL$}}{
                    $E\leftarrow E\setminus\{v-w\}$\\
                    $O\leftarrow O\cup\{w\to v\}$\\
                    $F\leftarrow$FALSE
                    }{
                        $E\leftarrow E\setminus\{v-w\}$\\
                        $O\leftarrow O\cup\{v\to w\}$\\
                        $F\leftarrow$FALSE
                    }
                }
            }
        \For{$u-v-w\in E$}{
            \If{$X^I_u\indep X^I_w,\forall I\in\I$}{
                $E\leftarrow E\setminus\{\{u,v\},\{v,w\}\}$\\
                $O\leftarrow O\cup\{u\to v, w\to v\}$\\
                $F\leftarrow$FALSE
            }
        }
    }
\end{algorithm}

\begin{algorithm}[ht]\caption{RECURSIVE COLLIDER\\$(\I,X,N,E,O,o)$}
\KwData{A set of intervention targets $\I$, a list of datasets $X$,
a list of sample sizes $N$, a list of unoriented edges $E$, a list of oriented edges $O$, an oriented edge $o=(u\rightarrow v)$, an indicator $COLL$ indicating the BIC score to use for the collider search when one edge has already been oriented.}
\KwResult{A list of oriented edges $O$, a list of unoriented edges $E$}
\SetKwFunction{RECOR}{RECURSIVE COLLIDER}
\label{alg.rec.vstruct}
\For{$(v,w)\in E$}{
    \eIf{$X^I_u\indep X^I_w,\forall I\in\I$\tcc*{using $COLL$}}{
       $E\leftarrow E\setminus\{v,w\}$\\
       $O\leftarrow O\cup\{ v\leftarrow w\}$
    }{
       $E\leftarrow E\setminus\{v,w\}$\\
       $O\leftarrow O\cup\{ v\rightarrow w \}$\\
       $E,O\leftarrow$\RECOR${(\I,X,N,E,O,o=(v\rightarrow w),COLL)}$
        }
    }
\end{algorithm}

\begin{algorithm}[ht]\caption{Procedure 1\\$(\I,X,N,V,E,PW,COLL)$}
\KwData{A set of intervention targets $\I$, a list of datasets $X$,
a list of sample sizes $N$, a list of vertices $V$, a list of undirected edges $E$, an indicator $PW$ that indicates if using the BIC scores or the test of the invariance of the regression coefficients, an indicator $COLL$ indicates the BIC score to use for the collider search when one edge has already been oriented.}
\KwResult{Interventional CPDAG $\mathcal{G}$}
\label{alg.proc.1}
$E,O\leftarrow$\Triplet${(\I,X,N,E,O=\emptyset,COLL)}$\\
$\mathcal{G}\leftarrow(V,O)$\\
\eIf{$PW==$"BIC"}{
    $\mathcal{G}\leftarrow(V,E)$\\
    $\mathcal{G}_1,\dots,\mathcal{G}_k\leftarrow$ connected components of $\mathcal{G}$\\
    \For{$i \in \{1,\dots,k\}$}{
        $\mathcal{G}_{d}\leftarrow $\FNDROOT$(\mathcal{G}_i)$\\
        $\mathcal{G}\leftarrow\mathcal{G}\cup\mathcal{G}_d$
    }
    $\mathcal{G}\leftarrow$ $\I$-CPDAG$(\mathcal{G})$
    }{
        \For{$e\in E$}{
                \If{$e$ \text{is directly} $\I$ identifiable}{
                    $o\leftarrow$ directed version of $e$ \tcc*{using $PW$}
                    $E,O\leftarrow$\RECOR$(\I,X,N,E,O,o,COLL)$
                }
            }
            $\mathcal{G}\leftarrow (V,(E\cup O))$
        }
\end{algorithm}

\begin{algorithm}[ht]\caption{Procedure 2\\$(\I,X,N,V,E,PW,COLL)$}
\KwData{A set of intervention targets $\I$, a list of datasets $X$,
a list of sample sizes $N$, a list of vertices $V$, a list of undirected edges $E$, an indicator $PW$ that indicated if using the BIC scores or the test of the invariance of the regression coefficients, an indicator $COLL$ indicated the BIC score to use for the collider search when one edge has already been oriented.}
\KwResult{Interventional CPDAG $\mathcal{G}$}
\label{alg.proc.2}
$O\leftarrow\emptyset$\\
    \For{$e\in E$}{
        \If{$e$ \text{is directly} $\I$ identifiable}{
            $o\leftarrow$ directed version of $e$ \tcc*{using $PW$}
            $E\leftarrow E\setminus \{e\}$\\
            $O\leftarrow O\cup \{o\}$
        }
    }
    $E,O\leftarrow$ \Triplet{$\I,X,N,E,O,COLL$}\\
    $\mathcal{G}\leftarrow (V,(E\cup O))$
\end{algorithm}

\clearpage
\section{Extended Simulation Studies}
\label{app.sec.add.exp}
In this section we present  further analysis of the behaviour of our algorithms.
\subsection{Description of Computational Settings}
\label{app.descr.comp}
When the generated graph is a polytree we first sample a random undirected trees by generating random Prüfer sequences \citep{prufer} and then independently orienting each edge, while the DAGs are generated either via Erd\H{o}s-R\'{e}nyi or Barabasi-Albert model \citep[Chap.~III,VII]{reka:2002}. For the coefficient matrices $\Lambda$, we draw each coefficient uniformly in $(-2, -0.5) \cup (0.5,2)$. To obtain samples, we multiply a Gaussian error vector with the matrix $(I-\Lambda)^{-T}$, the variances of the error vector are sampled uniformly in $[0.05,0.15]$. To ease the comparison with the algorithm, we use the same setup for simulate intervention that is used in \cite{Geenens:2022}. Each intervention changes the mean of the intervened variable in the following way: First a parameter $\hat{\mu}_{k,i}$ is sampled from $\mathcal{N}(0,2)$, then the interventional mean is set to $\mu^I_{k,i}=5\cdot\text{sign}(\hat{\mu}_{k,i})+\hat{\mu}_{k,i}$. The interventional variance is set to $0.5$ for all the intervened nodes. Other types of interventions behave similarly and can be seen in Section~\ref{app.sec.other_ints}. 
The performance of our learning procedures is measured by the structural Hamming distance (SHD) adapted to CPDAGs from \cite{Hauser:20212}. 
All code is available at \cite{github}. 

For all simulations, except the ones for Fig.~1c (left) in the main paper and  Section~\ref{app.sec.unbalanced},
the sample size of all interventional and observational data 
sets is $n/|\mathcal{I}|$ where $n$ is the total sample size. 
For settings where the size of the interventions is one ($k=1$), the intervention targets of all data sets are always different. If $k>1$, this is also the case, although it is possible that two different sets of intervention targets intersect nontrivially.

The baseline of the skeleton recovery, denoted baseObs in 
Fig.~1 of the main text,
corresponds to the Chow-Liu algorithm executed on the weight matrix of only the observational
data. The low-dimensional setting in Fig.~3 of the main text is 20 nodes and all one-node intervention targets. The high-dimensional setting in Fig.~3 of the main text is 500 nodes with 10 interventional datasets that each intervene on 10 randomly chosen nodes.
Each simulation was repeated 20 times in the low dimensional case and 10 times in the high-dimensional case.

\subsection{Skeleton recovery}
\label{app.subsec.skel.rec}
\emph{Pooled data:}
Fig.~\ref{fig.skeleton.flip.intervention} shows the performance of the three skeleton recovery procedures against two baselines; "pooled" that is computed using the Chow-Liu algorithm with the correlation matrix of the pooled data, and "baseObs" that is computed using observational data only. The "Flipped" intervention changes the sign of the regression coefficient. We observe that, when doing perfect interventions or interventions where the sign of the coefficients doesn't change, "pooled" exhibits superior performance than the methods using the aggregation functions. Pooling the data does not perform well when the intervention changes the sign of the coefficients, especially when the number of intervened nodes is high, this shows that the correlation matrix of the pooled data doesn't give rise to a $G$-valid matrix.

\emph{Running times:} Fig.~\ref{fig.skel_runtime} shows  the runtime of the skeleton recovery procedures in the high-dimensional setting ($p=500, n=1000$). The parameters in this figure are similar to those in Fig.~1 of the main text. As expected, mean and Itest are substantially faster than median since they only require arithmetic operations. \vspace{-0.1in}
\begin{figure}[ht!]
    \centering
    \includegraphics[scale =0.7]{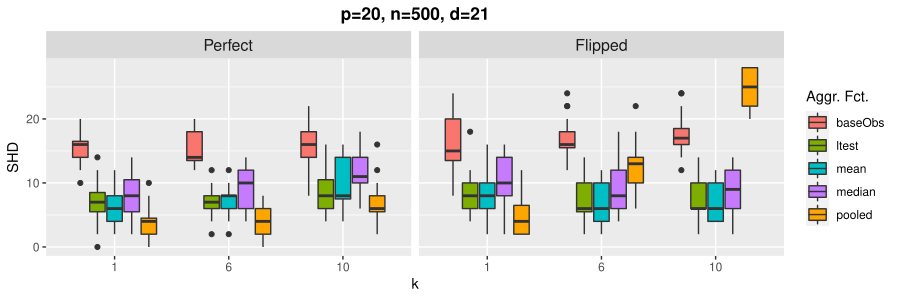}
    \caption{Skeleton recovery with perfect interventions (left) and interventions where the sign of the coefficient is flipped (right).}
    \label{fig.skeleton.flip.intervention}
        \includegraphics[scale =0.7]{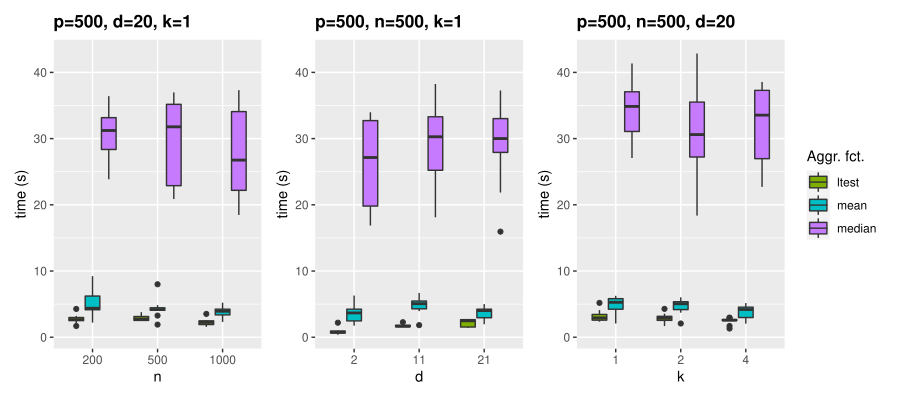}
    \caption{Skeleton recovery with base parameters 500 nodes, 1000 samples and 20 intervention targets with 10 intervened nodes each. The labels Itest, mean, median indicate each of the procedures in the list in  respectively.} 
    \label{fig.skel_runtime}
\end{figure}

\subsection{Analysis of Orientation Recovery}
\label{app.sec.orientation}

In this section we analyze the performance of the different orientation procedures P. 1 and P. 2. For each procedure we have two choices: pairwise orientation with either BIC (Section 4.3.2 of the main text) or IRC (Section 4.3.1 of the main text) and finding the colliders with either simple (Section 4.4.1 of the main text) or refined (Section 4.4.2 of the main text).
We analyze the performance of all eight possible procedures in 
Fig.~\ref{fig.ort_analysis}, where the four green scale colors indicate the four possibilities for P. 1 and the four blue 
scale colors indicate the four possibilities for P. 2. 

In all simulations we use the true underlying skeleton of the polytree.  
As a comparison  we record the mean SHD between the $\I$-CPDAG of the true polytree and a polytree obtained by randomly orienting each edge of the true skeleton. The plots in Fig.~\ref{fig.ort_analysis} show the behaviour of the orientation procedures in low- and high- dimensional settings for varying parameters $p=$number of nodes, $n=$ number of samples, $d=$ number of data sets (i.e. number of interventional datasets plus  the observational dataset), and $k=$number of nodes in each intervention.
All plots in Fig.~\ref{fig.ort_analysis} are for perfect interventions, other types of interventions behave similarly (see Section~\ref{app.sec.other_ints} and Fig.~\ref{fig.ori_diff_interv}).

With an overall glance at Fig.~\ref{fig.ort_analysis}, it seems like P. 2 (blue) behaves better than P. 1 (green). This is especially the case for increasing number of intervened nodes ($k$), c.f. plots (c) and (d). This can be explained since more causal information can be recovered by targeting more nodes in each intervention. Nevertheless, P.1 behaves comparably in settings with few interventions and thus should be further considered. Increasing the number of datasets ($d$), however, does increase SHD, c.f. plot (e) and (f). This could be explained by the observation that increasing the number of data sets for a fixed total sample size decreases the sample size of each individual data set which can lead to more inaccurate results.
Comparing refined and simple collider recovery, we notice that they have similar SHD, however, the simple methods are substantially faster, c.f. plots (g) and (h). Finally, we see that BIC has in general less SHD but is also substantially slower. 

Overall, it seems most reasonable to compare both P.1 and P.2 equipped with simple collider search and single edge orientation with IRC to guarantee accuracy as well as computational efficiency in high-dimensional settings.
\begin{figure}[p]
    \centering
    \includegraphics[scale =0.7]{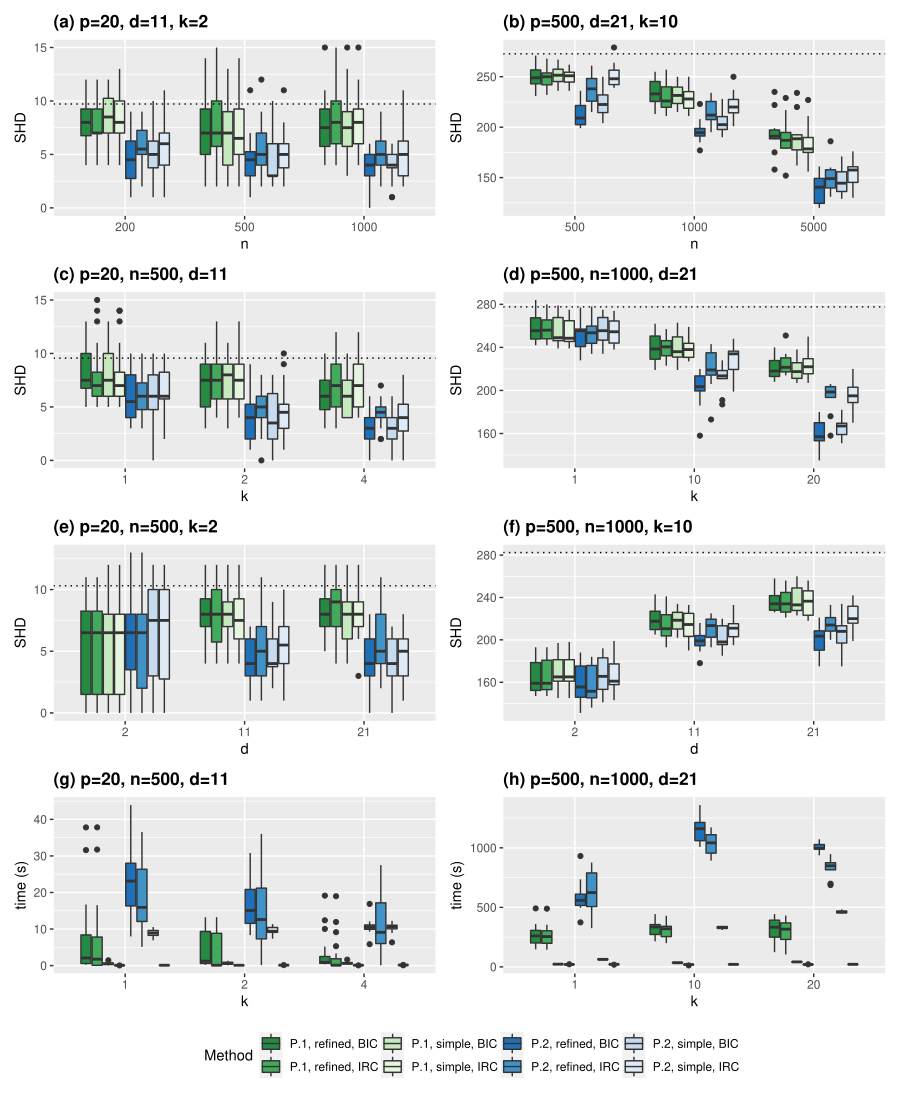}
    \caption{Orientation recovery from true skeleton in low- (left column) and high-dimensional (right column) setting with varying parameters and base parameters given in the title of each plot. The dotted line denotes the baseline of randomly orienting the edges.}
    \label{fig.ort_analysis}
\end{figure}
\subsection{Comparison with Other Types of Interventions (Orientation and Skeleton)}
\label{app.sec.other_ints}
In this section 
we compare the behaviour of our algorithms under different types of interventions, namely: ``Inhibitory"  that sets the regression coefficient $\lambda$ to $0.1\lambda$ and "Flipped" defined as in Section~\ref{app.subsec.skel.rec}. 

\emph{Skeleton recovery.} From Fig.~\ref{fig.skel_diff_interv}, we see that all the aggregation functions behave similarly under different intervention settings, except the pooling data under "Flipped" intervention that we already pointed out in Section~\ref{app.subsec.skel.rec}.

\emph{Orientation recovery:} From the plots in Fig.~\ref{fig.ori_diff_interv}, we see that there is no difference between ``Perfect" and ``Inhibitory" intervention while for "Flipped" intervention P. 2 performs significantly better.
\begin{figure}[!t]
    \centering
    \includegraphics[scale = 0.7]{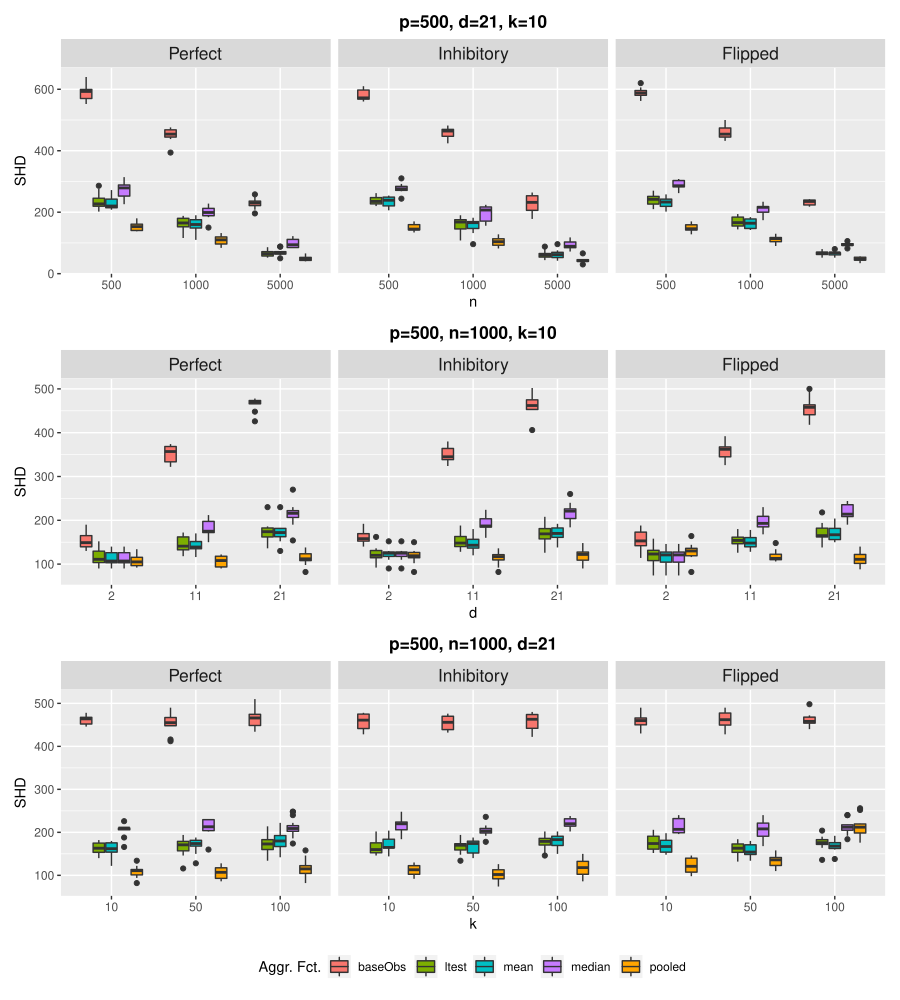}
    \caption{Skeleton recovery with the intervention types perfect (left), inhibitory (middle) and flipped (right) in high-dimensional setting.}
    \label{fig.skel_diff_interv}
\end{figure}

\begin{figure}[p]
    \centering
    \includegraphics[scale = 0.7]{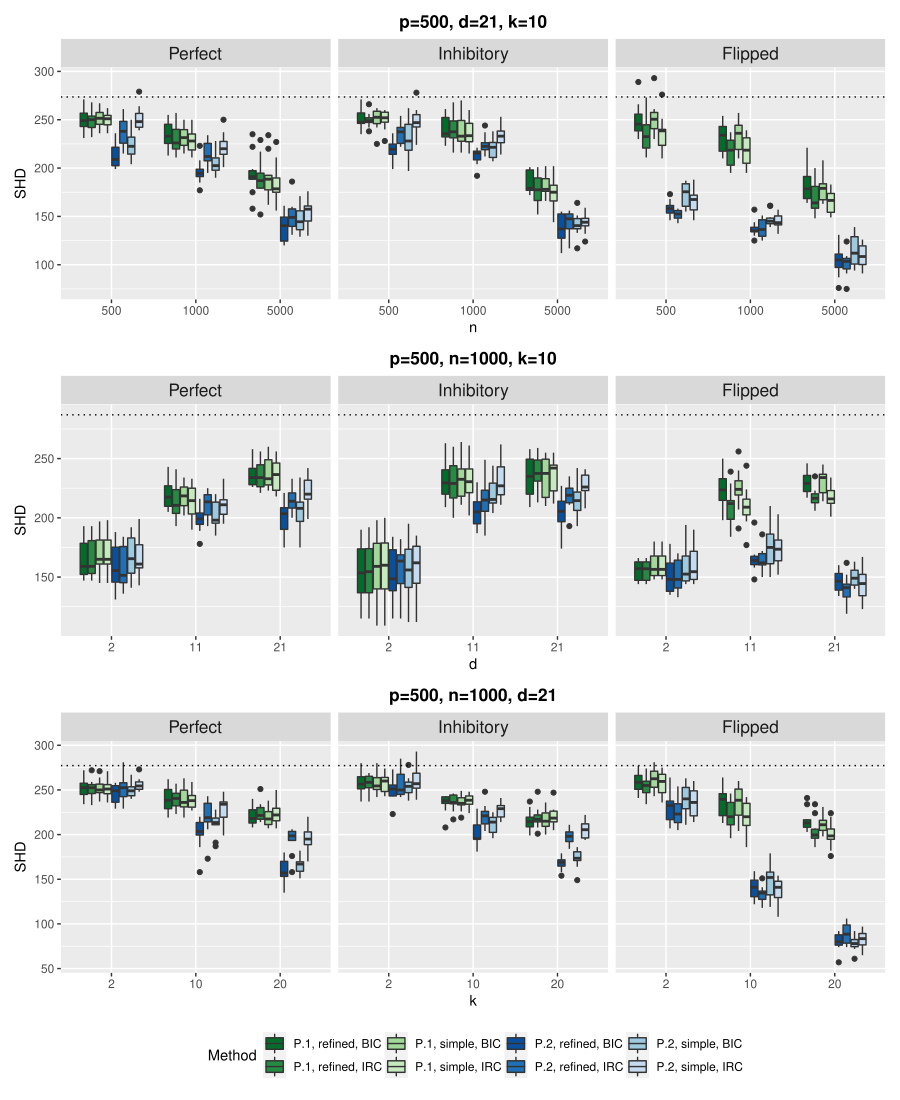}
    \caption{Orientation recovery from true skeleton in high-dimensional settings under perfect (left), inhibitory (middle) and flipped (right) interventions. The dotted line denotes the baseline of randomly orienting edges.}
    \label{fig.ori_diff_interv}
\end{figure}

\clearpage
\subsection{Further Comparisons with GIES}
\label{app.subsec.furth.comp}
In Section~V-B, Fig.~3 and Table~I of the main text,  we compared the performance of GIES with that of P. 1, simple, IRC and 
P. 2, simple, IRC. Here we expand  this analysis by comparing GIES and P. 2 simple, IRC on DAGs  with increasing expected number of edges per node, denoted by $e$. All the DAGs in the section are sampled from an Erd\H{o}s-R\'{e}nyi model. In terms of SHD, we see in Fig.~\ref{fig.dags_other_nbhs} that for $e=1,5$, P. 2, simple, IRC outperforms GIES. For increasing values of $e$ it is expected that P. 2, simple, IRC will get worse in terms of SHD because
 the true underlying DAG is not a polytree. We see however that its performance is comparable that
of GIES. The Table~\ref{app.tab.runtime.dags.high-dim} contains running times for these simulation.
For $e=1,5$  P. 2, simple, IRC is faster than GIES by two orders of magnitude and three for $e=10$, for $e=20$ the GIES algorithm did not finish after 24hrs hence it was aborted. 
\begin{figure}[!t]
    \centering
    \includegraphics[scale =0.7]{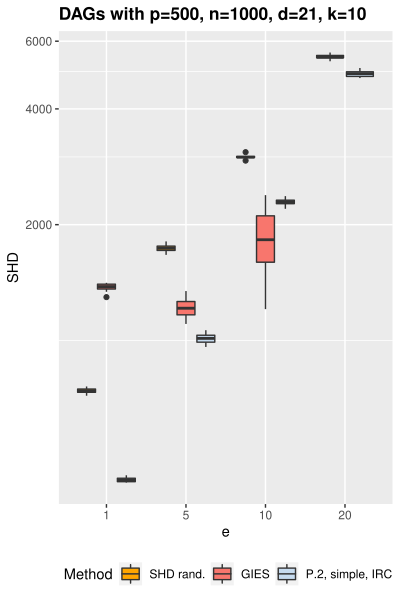}
    \caption{Recovery of the $\I$-CPDAG of DAGs with varying expected number of edges per node $e$ in high-dimensional setting. The box SHD rand. denotes the $\I$-CPDAG of randomly sampled polytrees. GIES on $e=20$ could not be computed in less than 24h.}
    \label{fig.dags_other_nbhs}
\end{figure}

\begin{table}[!t]
    \centering
    \setlength{\tabcolsep}{2pt}
    \caption{Mean and maximum runtime on DAGs with varying expected number of edges per node $e$ using P. 2, simple, IRC.}
        \begin{tabular}{l|r|r|r|r|r|r|r|r|}
        & \multicolumn{8}{c|}{Runtime in seconds} \\
        Method &  \multicolumn{2}{c|}{$e=1$}  &  \multicolumn{2}{c|}{$e=5$} &  \multicolumn{2}{c|}{$e=10$}&   \multicolumn{2}{c|}{$e=20$} \\
         & mean & max & mean & max &mean & max& mean& max \\
         \hline
        P. 2  & 7 & 9  & 7 & 9  & 8 &10 & 8 & 11\\
        GIES &  129 & 143 & 220 & 249 & 865 & 1249 & $>24$h&$>24$h \\
    \end{tabular}

    \label{app.tab.runtime.dags.high-dim}
\end{table}

\begin{table}[!t]
\centering
\caption{Mean and maximum runtime on low dimensional DAGs}
\begin{tabular}{l|r|r|r|}
& \multicolumn{2}{c|}{Runtime in seconds}\\
 Method &\multicolumn{2}{c|}{} \\
  &mean &max\\
  \hline
IGSP & 12.09 & 38.14 \\ 
  GIES & 0.03 & 0.06 \\ 
  P.1, simple, BIC & 0.32 & 0.91 \\ 
  P.1, simple, IRC & 0.05 & 0.07 \\ 
  P.2, simple, BIC & 6.01 & 7.69 \\ 
  P.2, simple, IRC & 0.07 & 0.08 \\ 
  P.1, refined, BIC & 8.40 & 81.23 \\ 
  P.1, refined, IRC & 7.94 & 78.47 \\ 
  P.2, refined, BIC & 12.37 & 26.79 \\ 
  P.2, refined, IRC & 12.92 & 65.46 \\ 
  DCDI & 18337.63 & 29762.54 \\ 
  dibs & 1354.14 & 1372.32 \\ 
\end{tabular}
\label{app.tab.runtime}
\end{table}

\subsection{Unbalanced Sample Sizes}
\label{app.sec.unbalanced}
Throughout the previous analyses, the sample size of each data set, both observational
and interventional, are all equal. In this section we present simulations for the case when the observational data set has more samples than the interventional data sets. In particular, we denote with $w_o$ the percentage of samples that are in the observational dataset, the samples that are not in the observational dataset are evenly distributed across the 20 interventional datasets. The scenario in which $w_o=0.048$ correspond to a balanced sample.

\emph{Skeleton recovery}: From the left plot in Fig.~\ref{app.fig.skeleton.unbalanced}, we see that the performance of Itest worsens when the data become more unbalanced, while the performance of median improves. Moreover, as predictable, we see that if the observational data is big enough the baseline that uses only observational data performs comparably with the other methods.

\emph{Orientation recovery:} From the right plot in Fig.~\ref{app.fig.skeleton.unbalanced} we see that P. 1 is unaffected from the unbalancing of the data, while the performances of P. 2 get slightly worse when the data are more unbalanced.

\begin{figure}[ht]
    \centering
    \hspace*{-1cm}                                                           
\includegraphics[scale =0.7]{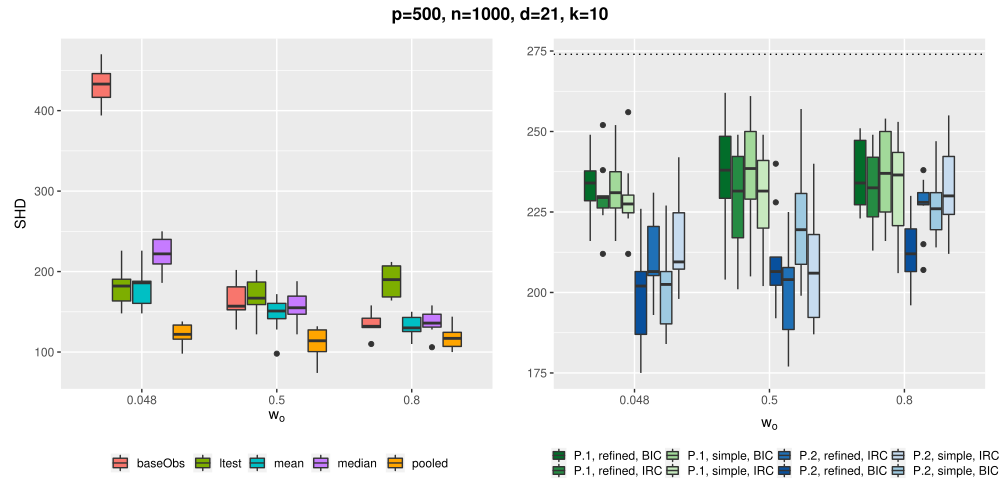}
\caption{Skeleton and orientation recovery  with unbalanced sample sizes. The dotted line denotes the baseline of randomly orienting the edges. The value $w_o$ is the percentage of samples that are in the observational dataset.}
\label{app.fig.skeleton.unbalanced}
\end{figure}

\emph{Comparison with GIES:}
From the plot in Fig.~\ref{app.fig.gies.unbalanced} we see that the performance of GIES is negatively affected from unbalancing of the data, while the performance of our algorithm is overall unaffected. The simulations are performed with DAGs of size $p=1000$ and expected number of edges per node $e=5$.

\subsection{Protein Expression Data}
Table~II in the main text shows the SHD between the consensus graph and the one produced by our algorithm. We show the two graphs in Fig.~\ref{fig:sachs:graph}.

\begin{figure}[ht!]
    \centering
        \centering
      \begin{tikzpicture}[scale = 0.5]
    [
        node distance=2cm,
        every node/.style={circle, draw=black, minimum size=1.5cm}, 
    ]
    
    \newcommand{\numnodes}{11}
    \newcommand{\nodenames}{
        praf,
        pmek,
        p44/42,
        plcg,
        PIP2,
        PIP3,
        pakts473,
        PKA,
        PKC,
        P38,
        pjnk
    }
    
    \foreach \i/\name in {1/praf, 2/pmek, 3/p4442, 4/plcg, 5/PIP2, 6/PIP3, 7/pakts473, 8/PKA, 9/PKC, 10/P38, 11/pjnk}
    {
        \node (\name) at ({360/\numnodes * (\i - 1)}:5cm) {\name};
    }
    
    \draw[->, blue] (praf) -- (pmek);
    \draw[->, blue] (pmek) -- (p4442);
    \draw[->, blue] (plcg) -- (PIP2);
    \draw[->, blue] (plcg) -- (PIP3);
    \draw[->, blue] (PIP2) -- (PKC);
    \draw[->, blue] (PIP3) -- (PIP2);
    \draw[->, blue] (PIP3) -- (pakts473);
    \draw[->, blue] (p4442) -- (pakts473);
    \draw[->, blue] (PKA) -- (praf);
    \draw[->, blue] (PKA) -- (pmek);
    \draw[->, blue] (PKA) -- (p4442);
    \draw[->, blue] (PKA) -- (pakts473);
    \draw[->, blue] (PKA) -- (P38);
    \draw[->, blue] (PKA) -- (pjnk);
    \draw[->, blue] (PKC) -- (PKA);
    \draw[->, blue] (pjnk) -- (PKC);
    \draw[->, blue] (pjnk) -- (PKA);
    \draw[->, blue] (PKA) -- (PKC);
    \draw[->, blue] (PKC) -- (PKA);
    
    \end{tikzpicture}
        \centering
                \begin{tikzpicture}[scale = 0.5]
    [
        node distance=2cm,
        every node/.style={circle, draw=black, minimum size=1.5cm}, 
    ]
    
    \newcommand{\numnodes}{11}
    \newcommand{\nodenames}{
        praf,
        pmek,
        p44/42,
        plcg,
        PIP2,
        PIP3,
        pakts473,
        PKA,
        PKC,
        P38,
        pjnk
    }
    
    \foreach \i/\name in {1/praf, 2/pmek, 3/p4442, 4/plcg, 5/PIP2, 6/PIP3, 7/pakts473, 8/PKA, 9/PKC, 10/P38, 11/pjnk}
    {
        \node (\name) at ({360/\numnodes * (\i - 1)}:5cm) {\name};
    }
    \draw[->, blue] (praf) -- (pmek);
    \draw[->, blue] (PIP2) -- (pmek);
    \draw[->, blue] (PIP2) -- (plcg);
    \draw[->, blue] (PIP3) -- (PIP2);
    \draw[->, blue] (p4442) -- (pakts473);
    \draw[->, blue] (pakts473) -- (PKA);
    \draw[->, blue] (PKA) -- (pmek);
    \draw[->, blue] (P38) -- (PKC);
    \draw[red] (P38) -- (pjnk);

        \end{tikzpicture}
    \caption{On the top, the consensus $I$-CPDAG for the Sachs data, and on the bottom, the one estimated by our algorithm. Different versions of our algorithm output the same result in this case.}
    \label{fig:sachs:graph}
\end{figure}
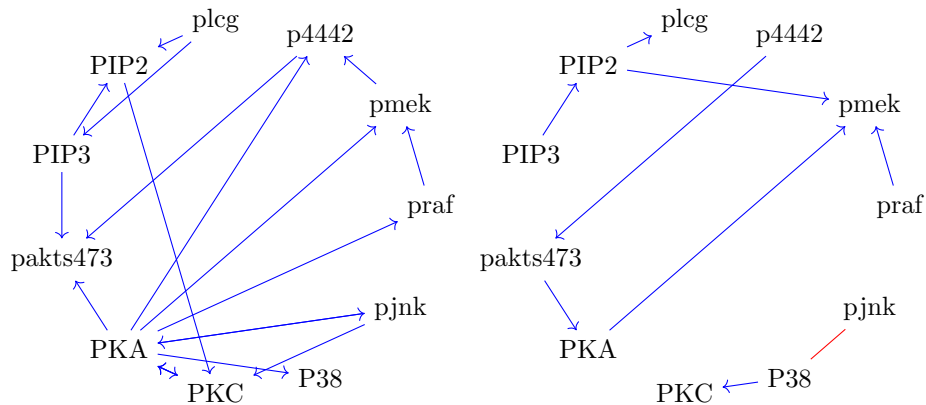

\begin{figure}[ht]
\centering
\includegraphics[scale =0.7]{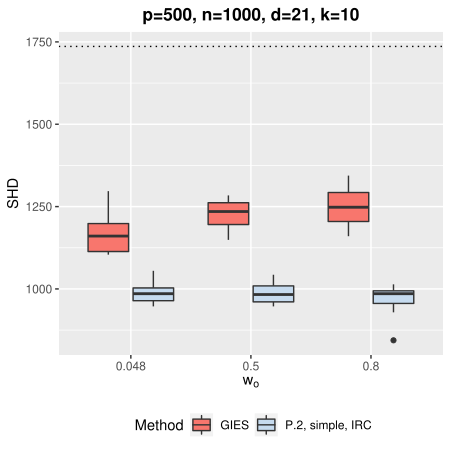}
\caption{Comparison between P. 2, simple, IRC and GIES on DAGS with expected number of edges per node $e=5$. The dotted line denotes the baseline of randomly orienting the edges.}
\label{app.fig.gies.unbalanced}
\end{figure}

\clearpage

\bibliography{ref}
\end{document}